%% file: main.tex
\documentclass{article}

\usepackage{microtype}
\usepackage{graphicx}
\usepackage{subfigure}
\usepackage{booktabs} 

\usepackage{hyperref}

\usepackage[preprint]{icml2026}

\usepackage{amsmath}
\usepackage{amssymb}
\usepackage{mathtools}
\usepackage{amsthm}
\usepackage{bbm}
\usepackage{xcolor}
\usepackage{soul}

\definecolor{Gray}{gray}{0.95}
\definecolor{Green}{HTML}{A4E28E} 
\definecolor{LightGreen}{HTML}{E0F6DE} 
\definecolor{Blue}{HTML}{BFC0FF} 
\definecolor{LightBlue}{HTML}{E7E6FF} 
\definecolor{DarkRed}{HTML}{C00000}
\definecolor{DarkGreen}{HTML}{3B7D23}

\usepackage[capitalize,noabbrev]{cleveref}
\input{math_commands.tex}

\usepackage{multirow}

\theoremstyle{plain}
\newtheorem{theorem}{Theorem}[section]
\newtheorem{proposition}[theorem]{Proposition}

\theoremstyle{definition}

\theoremstyle{remark}

\NewDocumentCommand{\xiusi}
{ mO{} }{\textcolor{cyan}{\textsuperscript{\textit{Xiusi}}\textsf{\textbf{\small[#1]}}}}

\NewDocumentCommand{\hongru}
{ mO{} }{\textcolor{red}{\textsuperscript{\textit{Hongru}}\textsf{\textbf{\small[#1]}}}}

\NewDocumentCommand{\bolian}
{ mO{} }{\textcolor{blue}{\textsuperscript{\textit{Bolian}}\textsf{\textbf{\small[#1]}}}}

\usepackage[textsize=tiny]{todonotes}

\usepackage{xspace}
\newcommand{\method}{\textsc{Flare}\xspace}
\newcommand{\methodfull}{\textsc{\textbf{Flare}} (\textsc{\textbf{F}}uture-aware \textsc{\textbf{l}}ook\textsc{\textbf{a}}head with \textsc{\textbf{r}}eward \textsc{\textbf{e}}stimation)\xspace}

\icmltitlerunning{A Planning-Centric Analysis of Long-Horizon Decision Making in LLM Agents}

\begin{document}

\twocolumn[
    \icmltitle{Why Reasoning Fails to Plan: A Planning-Centric Analysis of\\ Long-Horizon Decision Making in LLM Agents}



    \icmlsetsymbol{equal}{*}

    \begin{icmlauthorlist}
        \icmlauthor{Zehong Wang}{nd}
        \icmlauthor{Fang Wu}{stanford}
        \icmlauthor{Hongru Wang}{edinburgh}
        \icmlauthor{Xiangru Tang}{yale}
        \icmlauthor{Bolian Li}{purdue}
        \icmlauthor{Zhenfei Yin}{ox}
        \\
        \icmlauthor{Yijun Ma}{nd}
        \icmlauthor{Yiyang Li}{nd}
        \icmlauthor{Weixiang Sun}{nd}
        \icmlauthor{Xiusi Chen}{uiuc}
        \icmlauthor{Yanfang Ye}{nd}
    \end{icmlauthorlist}

    \icmlaffiliation{nd}{University of Notre Dame}
    \icmlaffiliation{stanford}{Stanford University}
    \icmlaffiliation{purdue}{Purdue University}
    \icmlaffiliation{edinburgh}{University of Edinburgh}
    \icmlaffiliation{ox}{University of Oxford}
    \icmlaffiliation{yale}{Yale University}
    \icmlaffiliation{uiuc}{UIUC}

    \icmlcorrespondingauthor{Zehong Wang}{zwang43@nd.edu}
    \icmlcorrespondingauthor{Yanfang Ye}{yye7@nd.edu}

    \icmlkeywords{Machine Learning, ICML}

    \vskip 0.3in
]



\printAffiliationsAndNotice{}  

\input{section/abstract}

\input{section/introduction}
\input{section/preliminary}
\input{section/analysis}
\input{section/methodology}

\input{section/experiments}
\input{section/related_work}
\input{section/conclusion}


\bibliography{citation}
\bibliographystyle{icml2026}


\newpage
\appendix
\onecolumn
\input{appendix/related_work}

\input{appendix/theory}
\input{appendix/method}

\input{appendix/experimental_setup}
\input{appendix/experimental_results}

\end{document}

%% file: math_commands.tex
\usepackage{amsmath,amsfonts,bm}

\def\eqref#1{equation~\ref{#1}}

\def\1{\bm{1}}

\DeclareMathAlphabet{\mathsfit}{\encodingdefault}{\sfdefault}{m}{sl}
\SetMathAlphabet{\mathsfit}{bold}{\encodingdefault}{\sfdefault}{bx}{n}



%% file: section/abstract.tex
\begin{abstract}
    Large language model (LLM)-based agents exhibit strong step-by-step reasoning capabilities over short horizons, yet often fail to sustain coherent behavior over long planning horizons.
    We argue that this failure reflects a fundamental mismatch: step-wise reasoning induces a form of step-wise greedy policy that is adequate for short horizons but fails in long-horizon planning, where early actions must account for delayed consequences.
    From this planning-centric perspective, we study LLM-based agents in deterministic, fully structured environments with explicit state transitions and evaluation signals.
    Our analysis reveals a core failure mode of reasoning-based policies: locally optimal choices induced by step-wise scoring lead to early myopic commitments that are systematically amplified over time and difficult to recover from.
    We introduce \methodfull as a minimal instantiation of future-aware planning to enforce explicit lookahead,  value propagation, and limited commitment in a single model, allowing downstream outcomes to influence early decisions.
    Across multiple benchmarks, agent frameworks, and LLM backbones, \method consistently improves task performance and planning-level behavior, frequently allowing LLaMA-8B with \method to outperform GPT-4o with standard step-by-step reasoning.
    These results establish a clear distinction between reasoning and planning.
\end{abstract}

%% file: section/introduction.tex
\section{Introduction}

Large language model (LLM)-based agents have demonstrated impressive step-by-step reasoning capability \citep{wei2022chain,wang2023selfconsistency,yao2022react,shinn2023reflexion,wu2024autogen,xi2025agentgymrl}.
However, growing empirical evidence suggests that models with strong reasoning capability does not work well on tasks requiring long-horizon decision making \citep{duan2024gtbench,jin2024impact}.
This discrepancy points to a fundamental conceptual gap: \emph{reasoning is not planning}.
From a planning perspective, we argue that standard LLM reasoning, e.g., chain-of-thought (CoT) \citep{wei2022chain}, can be viewed as a form of step-wise greedy policy based on local scores, e.g., next-step plausibility.
While such reasoning selects locally plausible actions, it cannot reshape early decisions according to their long-term consequences, since actions that are locally optimal may lead to poor outcomes over long horizons.
Conflating these two capabilities leads to systematic failures in long-horizon tasks.

In practice, this mismatch manifests as brittle long-horizon behavior. In unseen environments, small early mistakes often compound: trajectories grow longer than necessary \citep{yao2022react,huang2022inner}, decisions lose internal consistency \citep{yao2023tree,park2023generative}, and actions that appear locally reasonable fail when executed \citep{song2023llm,wu-etal-2024-reasoning}.
Although techniques like self-reflection \citep{shinn2023reflexion,madaan2023self}, bootstrapping \citep{wang2023self,zelikman2022star}, and self-improvement \citep{wang2024voyager} can reduce local errors, they remain within the same step-wise decision making, selecting actions based on local signals without explicit evaluation of outcomes.
This raises a central question: \emph{can LLM-based agents truly plan over long horizons, or can they only reason locally about the next step?}

Answering this question is challenging because long-horizon tasks entangle multiple sources of difficulty. They are consequence-driven \citep{song2023llm}, often provide delayed or sparse feedback \citep{xi2025agentgym,shridhar2021alfworld,yao2022webshop}, and are typically partially observable or stochastic \citep{ahn2022can}.
These factors obscure whether failures arise from environmental complexity or from limitations of the agent's decision mechanism itself.

To disentangle these effects, we adopt a controlled-variable diagnostic setting: deterministic environments with explicit state representations, known transition dynamics, and evaluative feedback available at planning time.
This setting removes environment-side uncertainty to isolate agent-side decision processes.
Within this controlled setting, our empirical and theoretical analyses reveal a consistent failure pattern.
Even when states and evaluation signals are explicit, agents guided by policies with step-wise signals face systematic failures in long-horizon planning, since they favor locally appealing actions that may lead to long-horizon failure.
Early deviations are then amplified over time, leading to rapid degradation as the planning horizon grows \citep{yu2020mopo,wei2022chain,besta2024graph}.
Crucially, this failure reflects a structural limitation of step-wise greedy decision making \citep{liu2023agentbench,cemri2025multi}: once actions are selected based on local signals, early commitments become difficult to revise, locking agents into irreversible trajectories \citep{xi2025agentgymrl,yao2023beyond}.

This diagnosis highlights a fundamental limitation of reasoning-based policies: improving local reasoning alone is insufficient to yield coherent long-horizon planning, regardless of whether it is obtained via prompting, beam search, or reinforcement learning.
Instead, coherent planning requires three minimal mechanisms: explicit lookahead, backward value propagation, and limited commitment through receding-horizon replanning.
Existing reasoning paradigms violate at least one of these requirements: for example, CoT and beam search \citep{snell2024scaling} lack explicit future evaluation and value propagation; Reflexion \citep{shinn2023reflexion} and ReAct \citep{yao2022react} lack explicit lookahead;
reinforcement learning \citep{guo2025deepseek} performs planning offline by embedding long-horizon decision-making into model parameters, therefore lacks online revision of early commitments.

Motivated by this insight, we propose \methodfull, a planning framework that minimally instantiates these design principles.
\method performs explicit lookahead to evaluate counterfactual future trajectories, propagates trajectory-level outcomes backward to inform early decisions, and commits only to the next action under a receding-horizon scheme.
Rather than introducing additional heuristics, \method serves as a canonical construction of future-aware planning, demonstrating how coherent long-horizon behavior can arise once these minimal mechanisms are enforced.
Across benchmarks, agent frameworks, and LLM backbones, \method consistently mitigates the long-horizon failures identified in our diagnosis, often allowing smaller models to outperform larger models with reasoning-based policies.
More broadly, our findings suggest that long-horizon failures of LLM-based agents reflect a fundamental gap between reasoning and planning.


Our main contributions are as follows:
\begin{itemize}
      \item \textbf{A planning-centric formulation of LLM reasoning.}
            We formalize LLM reasoning as a step-wise greedy decision policy, enabling a principled analysis of why strong local reasoning does not yield coherent long-horizon planning.

      \item \textbf{Mechanism-level diagnosis of reasoning-based decision making.}
            In deterministic, fully structured environments with explicit transitions and evaluation signals, we show that the pure reasoning induces early myopic commitments that are systematically amplified over time and are difficult to recover from.

      \item \textbf{Theoretical distinction between reasoning and planning.}
            We prove that step-wise greedy reasoning is arbitrarily suboptimal for long-horizon tasks, that increasing search width via beam search does not resolve this limitation, and that even one-step explicit lookahead suffices to strictly improve decision capability.

      \item \textbf{A minimal planning mechanism for LLM agents.}
            We propose \method as a minimal instantiation of future-aware planning that combines explicit lookahead, value propagation, and limited commitment, and demonstrate consistent improvements across benchmarks and agent frameworks.
\end{itemize}

%% file: section/preliminary.tex
\section{Problem Definition}

We study the relationship between planning and reasoning in LLM-based agents from a planning-first perspective.
We formalize LLM reasoning as a specific class of planning strategy: \emph{step-wise greedy policy}, referring to \emph{reasoning-based policy}, where the agent selects actions based solely on local evaluation at the current state.
This formulation allows us to compare reasoning-based policies with planning mechanisms that incorporate explicit long-term evaluation 
under explicit environments.

To enable such a comparison, we consider a controlled environment in which transition dynamics and evaluative feedback are explicit at planning time, isolating the agent's decision mechanism from environment-side uncertainty.


\noindent\textbf{Environment.}
We model the environment as a deterministic state-transition system
$\mathcal{G} = (\mathcal{S}, \mathcal{A}, T, \hat{r})$,
where $\mathcal{S}$ is a discrete state space, $\mathcal{A}$ an action space, and
$T: \mathcal{S} \times \mathcal{A} \rightarrow \mathcal{S}$ the transition function.
At state $s_t$, the agent selects $a_t \in \mathcal{A}(s_t)$, inducing a trajectory
\begin{equation}
    \tau = (s_0, a_0, s_1, \dots, s_H), \quad s_{t+1}=T(s_t,a_t).
\end{equation}
An evaluative signal $\hat{r}: \mathcal{S}\times\mathcal{A}\times\mathcal{S}\rightarrow\mathbb{R}$ is available at planning time, and the cumulative return is
\begin{equation}
    R(\tau)=\sum_{t=0}^{H-1}\hat{r}(s_t,a_t,s_{t+1}).
\end{equation}

\noindent\textbf{Agent Policy.}
We abstract agent behavior as a policy $\pi:\mathcal{S}\rightarrow\mathcal{A}$.
A reasoning-based policy selects
\begin{equation}
    \pi_{\text{greedy}}(s)=\arg\max_{a\in\mathcal{A}(s)} \hat{r}(s,a,T(s,a)),
\end{equation}
corresponding to step-wise greedy planning based on local evaluation.
Beam search extends this paradigm by selecting among the top-$k$ actions, while lookahead incorporates limited trajectory rollouts to estimate future return.

\noindent\textbf{Problem Statement.}
We ask whether such step-wise greedy planning strategies (i.e., reasoning) can support coherent long-horizon decision making, even when full state information, transition dynamics, and evaluative feedback are explicitly available.
If not, this establishes a structural distinction between reasoning-based and planning-based decision making, demonstrating that planning capability cannot be reduced to step-wise reasoning.

%% file: section/analysis.tex
\section{Failure Modes of Long-Horizon Planning}
\label{sec:diagnosis}

\begin{figure}[!t]
    \centering
    \subfigure{\includegraphics[width=0.49\linewidth]{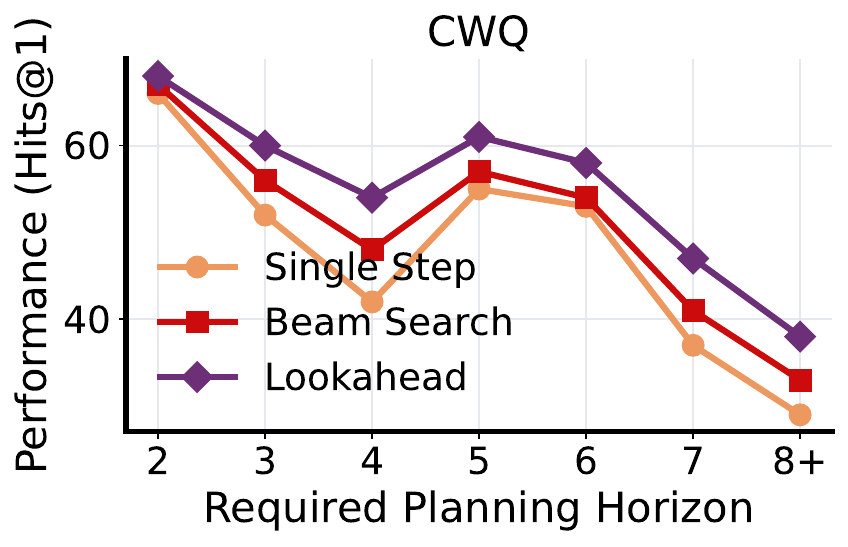}}
    \subfigure{\includegraphics[width=0.49\linewidth]{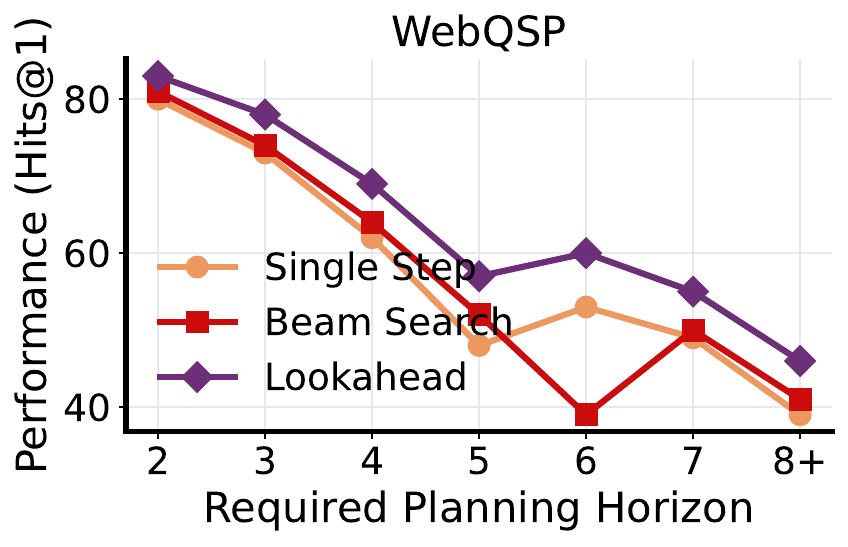}}

    \subfigure{\includegraphics[width=0.49\linewidth]{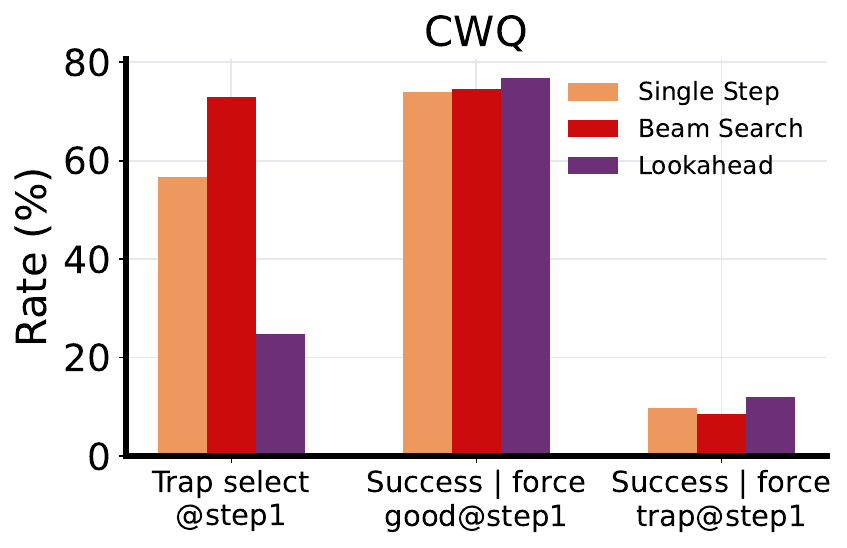}}
    \subfigure{\includegraphics[width=0.49\linewidth]{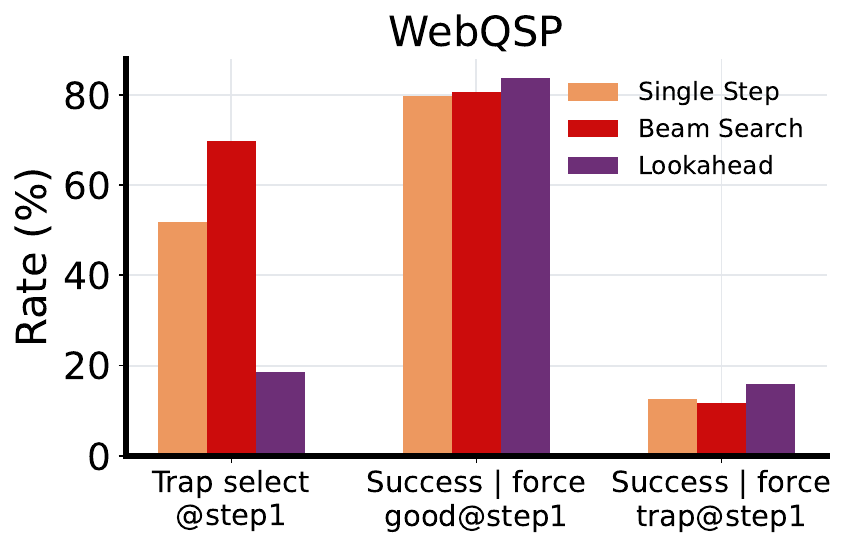}}
    \caption{
        \textbf{Planning performance and early myopia under increasing horizons.}
        \textit{Top}: Accuracy (Hits@1) on CWQ and WebQSP by required planning horizon. Performance degrades rapidly for single step (i.e., reasoning-based policy) and beam search as the horizon increases, while lookahead degrades more slowly.
        \textit{Bottom}: Myopic trap selection at the first decision step. Step-wise scoring (single step and beam search) frequently selects actions with high immediate scores but poor long-term outcomes, whereas lookahead suppresses such traps via future-aware evaluation.
    }
    \label{fig:diagnosis-1}
\end{figure}

We show that long-horizon failures of pure reasoning arise from structural limitations of reasoning-based policy itself, rather than from environment uncertainty or model capacity.
In deterministic and fully structured environments, reasoning-based policies systematically introduce early myopic deviations that are amplified over time, leading to irreversible trajectory collapse.
We characterize this mechanism through performance degradation, local decision bias, the position of the first error, and post-error recoverability.
Incorporating explicit lookahead with trajectory-level evaluation substantially alleviates all observed failure modes.

\subsection{Empirical Analysis of Planning Failures}
\label{sec:diagnosis empirical}

We instantiate long-horizon planning using knowledge graph question answering (KGQA) \citep{chen2024plan,sun2024thinkongraph} as an evaluation sandbox, where agent behavior corresponds to multi-hop graph traversal.
Experiments are conducted on CWQ \citep{talmor2018web} and WebQSP \citep{yih2016value}, which exhibit different shortest-path length distributions and enable analysis across varying horizons.
To isolate agent-side decision mechanisms, we adopt the oracle-structure setting \citep{sun2024thinkongraph}, in which each query is paired with a subgraph guaranteed to contain a valid solution path.
Across all experiments, we fix the LLM backbone\footnote{LLaMA 3.1 8B by default \citep{grattafiori2024llama}.}, agent architecture, and evaluation signal, and vary only the planning strategy.
We consider three paradigms: (i) reasoning-based policy via greedy step-wise selection (single step), (ii) width-extended reasoning (beam search), and (iii) shallow planning-based policy with rollout evaluation (lookahead).
Results are shown in Figure~\ref{fig:diagnosis-1} and Figure~\ref{fig:diagnosis-2}, and additional details are provided in Appendix \ref{sec:experimental}.

\begin{figure}[!t]
    \centering

    \subfigure{\includegraphics[width=0.49\linewidth]{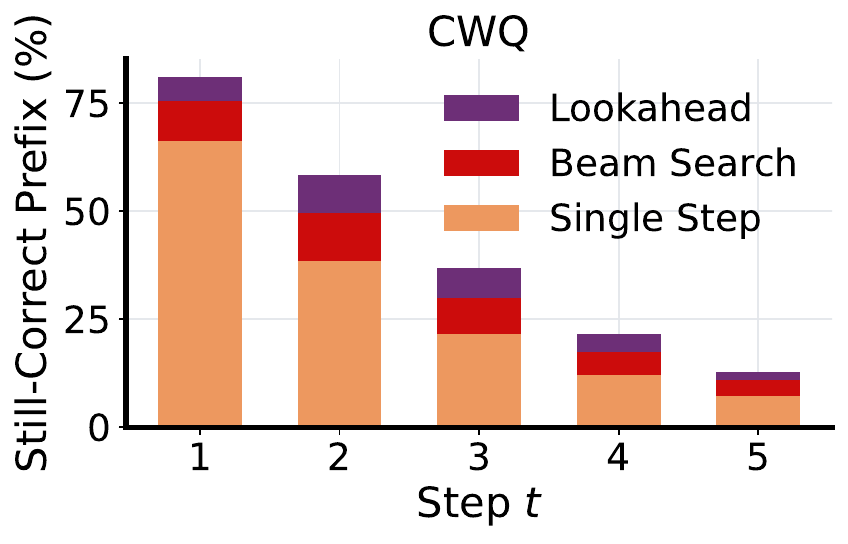}}
    \subfigure{\includegraphics[width=0.49\linewidth]{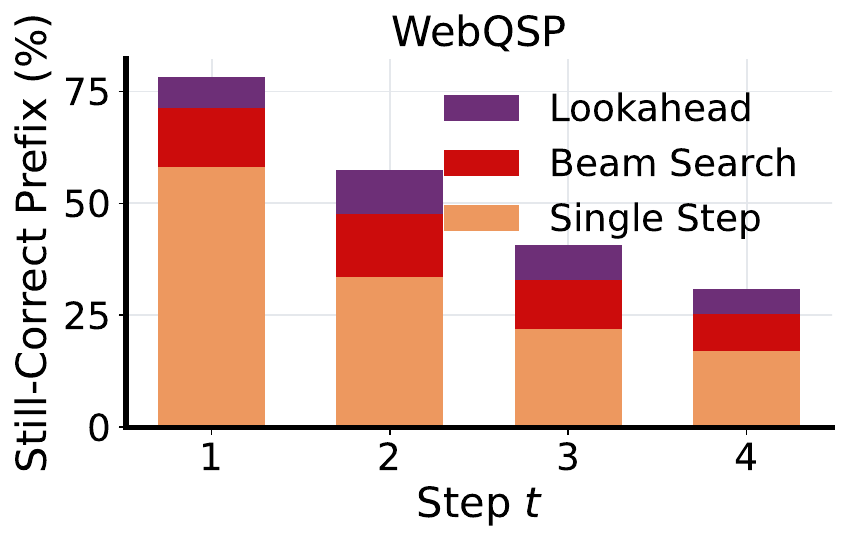}}

    \subfigure{\includegraphics[width=0.49\linewidth]{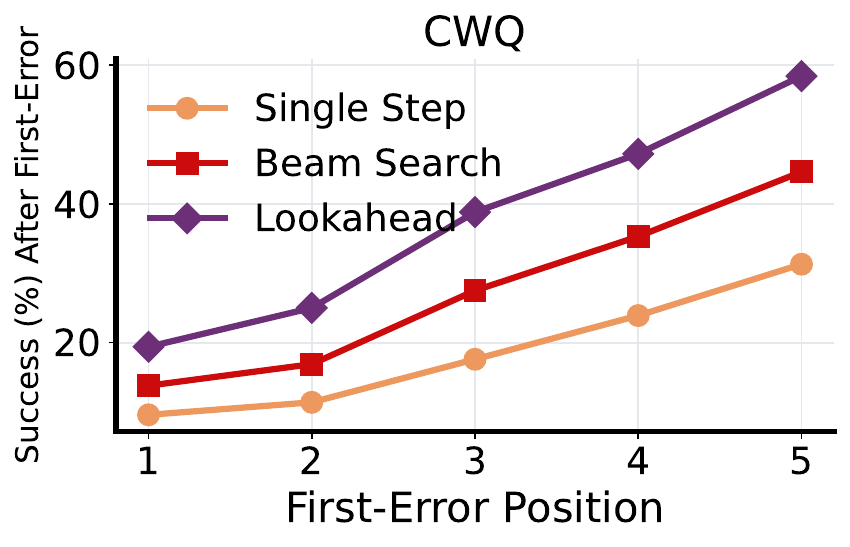}}
    \subfigure{\includegraphics[width=0.49\linewidth]{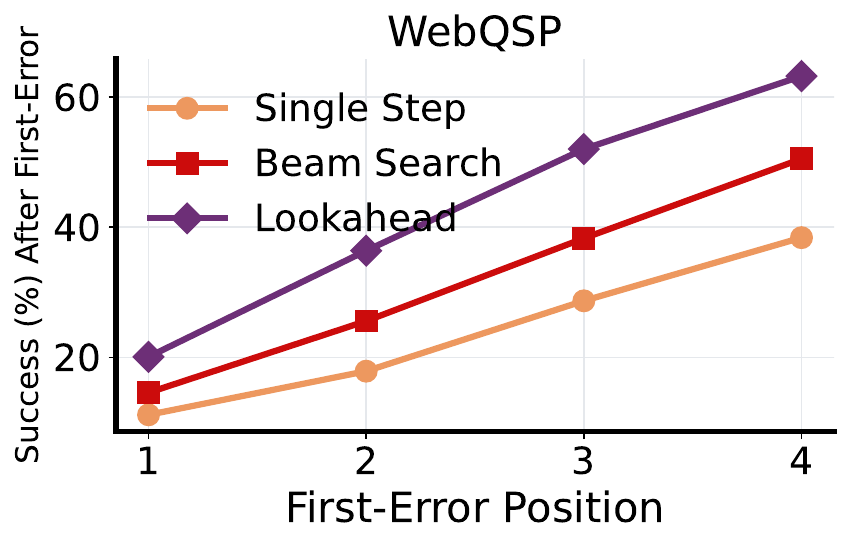}}
    \caption{
        \textbf{Error accumulation and recovery dynamics across decision paradigms.}
        \textit{Top}: Fraction of trajectories remaining on a correct prefix over decision steps. Step-wise scoring (single step and beam search) policies diverge early and irreversibly, while lookahead preserves correct prefixes substantially longer.
        \textit{Bottom}: Success rate conditioned on the position of the first error. After early deviations, single step (i.e., reasoning-based policy) and beam search policies rarely recover, whereas Lookahead enables consistent recovery across all error positions.
    }
    \label{fig:diagnosis-2}
\end{figure}

\noindent\textbf{Observation 1: Long-horizon collapse is a failure of reasoning-based policy.}
Across datasets, accuracy decreases rapidly as the required planning horizon increases (Figure~\ref{fig:diagnosis-1}, top).
Single step degrades sharply beyond shallow instances, and beam search provides only limited and transient improvement.
In contrast, lookahead consistently maintains higher accuracy and degrades more slowly.
This indicates that long-horizon failure is not merely a consequence of longer trajectories, but of relying on step-wise greedy policies.

\noindent\textbf{Observation 2: Reasoning-based policies induce systematic early myopia.}
We construct \emph{myopic traps}: actions that are locally optimal under step-wise scoring but lead to globally suboptimal regions.
At the first decision step, both greedy and beam-based strategies select such actions at high rates (Figure~\ref{fig:diagnosis-1}, bottom), indicating that local signals fail to encode downstream utility.
Lookahead substantially suppresses this bias by incorporating future rollouts.
Causal interventions confirm this mechanism: enforcing a correct first action equalizes performance across strategies, while enforcing a trap action causes universal failure.
Thus, early myopic commitment is the dominant source of long-horizon degradation in reasoning-based policy.

\noindent\textbf{Observation 3: Errors under reasoning-based policy are effectively irreversible.}
Reasoning-based and beam-based policies typically deviate from optimal trajectories within the first few decisions and rarely recover thereafter (Figure~\ref{fig:diagnosis-2}).
Increasing search width delays, but does not prevent, irreversible commitment to suboptimal prefixes.
In contrast, lookahead both postpones the first error and substantially improves recovery across all stages.
This shows that future-aware evaluation (i.e., planning-based policy) is necessary not only to avoid early mistakes, but to enable corrective revisions once errors occur.

\subsection{Fundamental Limits of Reasoning-based Policy in Long-Horizon Planning}

We establish three theoretical results that formalize the mechanism underlying the empirical failures in Section \ref{sec:diagnosis empirical}.
To this end, we abstract the decision signal available to the agent as a step-wise scoring function
\begin{equation}
    \hat{u}: \mathcal{S} \times \mathcal{A} \rightarrow \mathbb{R},
\end{equation}
which assigns a scalar score to each state--action pair.
Importantly, we make no assumption that $\hat{u}(s,a)$ is a consistent estimator of the optimal action-value function or that it encodes cumulative future return.
Complete proofs are deferred to Appendix~\ref{sec:proof}.

\noindent\textbf{Result 1: 
Step-wise greedy policy is arbitrarily suboptimal.
}
We begin with the decision paradigm underlying LLM reasoning: selecting actions greedily according to a step-wise surrogate score $\hat{u}$.
Despite its apparent rationality at each step, we show that this paradigm is fundamentally incompatible with reliable long-horizon planning.

\begin{proposition}
    \label{prop:greedy}
    Let $\pi_{\text{greedy}}$ be a step-wise greedy policy defined by
    $\pi_{\text{greedy}}(s) \in \arg\max_{a \in \mathcal{A}(s)} \hat{u}(s,a)$.
    Suppose there exists a state $s \in \mathcal{S}$ and two actions $a,b \in \mathcal{A}(s)$ such that
    $\hat{u}(s,a) > \hat{u}(s,b)$.
    Then for any finite horizon $H \ge 2$ and any constant $M > 0$, there exists a deterministic environment
    $\mathcal{G} = (\mathcal{S}, \mathcal{A}, T, \hat{r})$, consistent with $\hat{u}$, as
    \begin{equation}
        R(\tau(\pi_{\text{greedy}})) = 0,
        \qquad
        \max_{\pi} R(\tau(\pi)) \ge M.
    \end{equation}
    Consequently, the suboptimality gap of $\pi_{\text{greedy}}$ can be made arbitrarily large.
\end{proposition}

Proposition~\ref{prop:greedy} shows that the reasoning-based policy can exhibit arbitrarily poor long-horizon behavior.
Even when a trajectory with unbounded return exists, such policies can be forced into zero-return outcomes by committing to an early locally preferred action.
This provides a formal explanation for the myopic traps observed empirically: once decisions are guided solely by local scores, catastrophic deviation becomes unavoidable in the worst case.

\noindent\textbf{Result 2: Increasing search width does not overcome the limitations of reasoning-based policy.}
A natural response to the failure of step-wise greedy policy is to expand the search width and maintain multiple candidate trajectories.
Beam search therefore represents a stronger form of reasoning-based policy, but still ranks prefixes using step-wise surrogate scores.
We show that this modification does not fundamentally change the underlying decision paradigm.

\begin{proposition}
    \label{prop:beam}
    Consider a beam search policy that maintains at most $B \in \mathbb{N}$ partial trajectories
    and ranks prefixes solely by the accumulated surrogate score
    \begin{equation}
        S_t = \sum_{i=0}^{t-1} \hat{u}(s_i,a_i).
    \end{equation}
    For any beam width $B \in \mathbb{N}$, any horizon $H \ge 2$, and any constant $M > 0$,
    there exists a deterministic environment $\mathcal{G}$ and a surrogate function $\hat{u}$ such that
    (i) there exists a unique optimal trajectory $\tau^*$ with return $R(\tau^*) = M$; and
    (ii) the beam-search policy necessarily prunes the prefix of $\tau^*$ at depth $1$, yielding final return $R(\tau)=0$.
\end{proposition}

For any fixed beam width, there exist environments in which the optimal trajectory is irrevocably pruned after the first decision, before any informative future reward is observed.
Beam search thus delays commitment by expanding width, but cannot prevent premature elimination of globally optimal paths when ranking is driven by local scores.
This aligns with the empirical diagnosis that beam search mitigates but does not resolve long-horizon failures.

\noindent\textbf{Result 3: Even minimal lookahead (planning-based policy) strictly expands planning capability.}
We now show that the above limitations are not inherent to finite-horizon decision making itself, but to the absence of explicit future-aware planning.
Introducing even a single step of lookahead fundamentally changes what an agent can achieve in the worst case.

\begin{proposition}
    \label{prop:lookahead}
    Consider a lookahead policy that, at state $s$, evaluates each candidate action
    $a \in \mathcal{A}(s)$ by simulating trajectories of depth at most $k \ge 1$
    under the true transition function $T$, and selects the action maximizing the resulting
    cumulative return estimate.
    Then there exists a class of deterministic environments in which
    (i) any policy that selects actions solely based on step-wise surrogate scores attains zero return; and
    (ii) a lookahead policy with $k = 1$ achieves the optimal return.
    Consequently, even one-step lookahead strictly dominates step-wise decision-making
    in terms of worst-case performance guarantees.
\end{proposition}

Proposition~\ref{prop:lookahead} establishes a strict distinction between reasoning- and planning-based policies.
There exist environments in which all policies based solely on step-wise surrogate scores provably fail, while a policy with one-step lookahead achieves the optimal return.
This formalizes the intuition that planning capability arises from explicitly evaluating downstream consequences, rather than from stronger reasoning or wider search alone.
We further analyze how lookahead performance degrades under delayed rewards and truncated horizons in Appendix~\ref{sec:suboptimality}.

\subsection{Implications for Long-Horizon Decision Making}

Our empirical and theoretical results jointly support a planning-first conclusion: \emph{planning capability cannot be reduced to step-wise reasoning}.
Reasoning-based policies optimize local scores and therefore make irreversible early commitments; empirically, these commitments amplify into long-horizon collapse, and theoretically, they can be arbitrarily suboptimal.
Beam search increases width but preserves the same step-wise scoring mechanism, and thus cannot provide worst-case planning guarantees.
Consequently, improving reasoning accuracy or scaling step-wise search is insufficient for reliable long-horizon planning.

By contrast, planning-based policy strictly expands what agents can achieve.
Even one-step lookahead avoids myopic traps in environments where all step-wise strategies provably fail.
These results point to three minimal requirements for coherent long-horizon planning:
\textbf{(1) explicit future evaluation}, \textbf{(2) backward value propagation}, and \textbf{(3) limited commitment}.
Our \method instantiates these requirements.

%% file: section/methodology.tex
\section{\method for Long-Horizon Planning}

We present \method, a planning mechanism for LLM-based agents.
In contrast to reasoning-based policies, which select actions step by step using local scores, \method evaluates actions through their long-term consequences and propagates trajectory-level feedback backward to guide early decisions.
Concretely, \method instantiates two minimal components of planning:
(i) explicit lookahead via trajectory simulation, and
(ii) backward value propagation.
We implement these components using Monte Carlo Tree Search (MCTS) \citep{silver2018general} under a receding-horizon scheme \citep{mattingley2011receding}, committing only to the next action and replanning after each transition to avoid brittle long-term commitment under noisy evaluation.

While recent work has integrated MCTS into LLM-based agents \citep{hao2023reasoning,zhang2024rest,yoon2025monte,wang2025damr,wu2025deepsearch}, our contribution is conceptual rather than algorithmic: we show that explicit planning mechanisms are necessary for long-horizon coherence, whereas step-wise reasoning (even when augmented with wider search) remains structurally insufficient.
Additional details are provided in Appendix~\ref{sec:technical details}. 


\subsection{Explicit Lookahead}

A central failure of reasoning-based policies is that actions are evaluated using only local surrogate scores.
\method replaces this paradigm with future-aware evaluation: each candidate action is assessed by simulating the trajectories it induces, and decisions are guided by trajectory-level outcomes rather than step-wise signals.

\noindent\textbf{Selection.}
\method maintains a search tree rooted at the current state and allocates simulation effort using
\begin{equation}
    \label{eq:select}
    a_t = \arg\max_{a \in \mathcal{A}(s_t)}
    \left(
    Q(s_t,a) + c \sqrt{\frac{\log N(s_t)}{N(s_t,a) + 1}}
    \right),
\end{equation}
where $Q(s_t,a)$ aggregates returns from simulated trajectories passing through $(s_t,a)$ and $N(s_t,a)$ is the visit count.
This prioritizes branches with favorable long-term outcomes, enabling early decisions to reflect downstream consequences rather than local heuristics.

\noindent\textbf{Expansion with Action Pruning.}
Tree expansion brings new successor states when previously unexplored state--action pairs are encountered.
To control computational cost, expansion is restricted to a bounded candidate set
\begin{equation}
    \mathcal{A}_k(s) = \phi(s), \qquad |\mathcal{A}_k(s)| \le k,
\end{equation}
where $\phi$ proposes feasible actions (implemented by an LLM in our experiments).
Importantly, action pruning only limits which futures are evaluated and does not influence action values or preferences.

\input{table/plan-strategy-compare}

\subsection{Value Propagation}

Lookahead alone is insufficient unless information from simulated futures can revise early decisions.
\method therefore enforces backward value propagation: trajectory-level outcomes are aggregated and propagated to update the estimated quality of earlier actions, preventing irreversible commitment based on local evidence.

\noindent\textbf{Trajectory-Level Evaluation with Trajectory Memory.}
Each simulated trajectory
\(
\tau = (s_0, a_0, \dots, s_H)
\)
is assigned a cumulative return
\begin{equation}
    \label{eq:return}
    R(\tau) = \sum_{t=0}^{H-1} \hat{r}(s_t, a_t, s_{t+1}),
\end{equation}
where $\hat{r}$ is an evaluative signal available at planning time.
In our implementation, $R(\tau)$ is not computed by directly summing step-wise rewards.
Instead, we use an LLM to score trajectories by comparing multiple sampled candidates and assigning relative preference-based evaluations.
Actions are therefore assessed exclusively through their estimated downstream consequences rather than step-wise scores.

To amortize evaluation cost, \method maintains a bounded trajectory memory
\(
\mathcal{M} = \{(\tilde{\tau}, R(\tilde{\tau}))\}.
\)
For a new trajectory $\tau$, the most similar cached trajectory
\begin{equation}
    \label{eq:retrieve}
    \tilde{\tau}^* = \arg\max_{\tilde{\tau} \in \mathcal{M}} \mathrm{sim}(\tau, \tilde{\tau})
\end{equation}
is reused when $\mathrm{sim}(\tau, \tilde{\tau}^*) \ge \delta$; otherwise $R(\tau)$ is computed directly.
The memory size is bounded by $|\mathcal{M}| \le M$.

\noindent\textbf{Backward Value Update.}
For each simulated trajectory generated on fixed-depth ($H=3$ by default), its return $R(\tau)$, computed over the simulated horizon, is propagated backward to all visited state--action pairs. The action-value function $Q(s_t,a_t)$ is updated by aggregating returns from all simulations that pass though $(s_t,a_t)$.
Action selection at the root follows
\begin{eqnarray}
    a^* = \arg\max_{a \in \mathcal{A}(s)} Q(s,a),
\end{eqnarray}
so early decisions are determined by accumulated long-term evidence rather than step-wise heuristics.



\subsection{Discussion}

The receding-horizon design of \method follows directly from the failure mechanism identified in our diagnosis.
Even under full observability, offline planning is brittle when evaluation signals are imperfect, as is typical for LLM-based verifiers: committing to an entire action sequence propagates early estimation errors throughout the trajectory.
By committing only to the next action and replanning after each transition, \method limits the impact of early mistakes and allows later trajectory-level evidence to revise earlier action values.
This design also clarifies why \method is not restricted to our controlled setting.
The method does not require accurate long-term prediction or exhaustive enumeration of future states; it only assumes that candidate futures can be simulated and evaluated at planning time.
These conditions hold in a wide range of partially observed environments
Thus, although our analysis is conducted in deterministic environments to isolate decision mechanisms, the underlying principle (future-aware evaluation under limited commitment) extends beyond this setting.

%% file: table/plan-strategy-compare.tex
\begin{table*}[!t]
    \caption{
        \textbf{Performance comparison of planning strategies across benchmarks, LLM backbones, and agent frameworks.}
        Results are reported on CWQ, WebQSP, and GrailQA using identical base frameworks and LLM backbones, with the best performance highlighted in \textbf{bold}.
        \method achieves the strongest overall performance, with larger gains on more challenging datasets.
    }
    \label{tab:compare-planning}
    \resizebox{\linewidth}{!}{
        \begin{tabular}{cl| cccc | cccc | c}
            \toprule
                                       &                        & \multicolumn{4}{c|}{\textbf{Think-on-Graph} \citep{sun2024thinkongraph}} & \multicolumn{4}{c|}{\textbf{Plan-on-Graph} \citep{chen2024plan}} & \multicolumn{1}{l}{}                                                                                                            \\ \cmidrule(lr){3-6}\cmidrule(lr){7-10}
            \textbf{Benchmark}         & \textbf{Plan Strategy} & \textbf{LLaMA}                                                           & \textbf{LLaMA}                                                   & \textbf{GPT-4o}      & \textbf{GPT-4o} & \textbf{LLaMA} & \textbf{LLaMA} & \textbf{GPT-4o} & \textbf{GPT-4o} & \textbf{Average} \\
                                       &                        & \textbf{8B}                                                              & \textbf{70B}                                                     & \textbf{mini}        & \textbf{}       & \textbf{8B}    & \textbf{70B}   & \textbf{mini}   & \textbf{}       &                  \\ \midrule
            \multirow{4}{*}{{CWQ}}     & Single Step            & 46.9                                                                     & 50.5                                                             & 51.9                 & 58.2            & 65.8           & 69.7           & 66.7            & 68.6            & 59.8             \\
                                       & Beam Search            & 51.3                                                                     & 53.7                                                             & 58.9                 & 63.5            & 66.7           & 72.1           & 68.7            & 71.3            & 63.3             \\
                                       & Lookahead              & 52.7                                                                     & 56.5                                                             & 60.0                 & 65.8            & 71.7           & 75.6           & 75.0            & 74.7            & 66.5             \\
                                       & \method                & \textbf{58.1}                                                            & \textbf{63.2}                                                    & \textbf{69.9}        & \textbf{73.6}   & \textbf{73.9}  & \textbf{79.1}  & \textbf{77.6}   & \textbf{78.8}   & \textbf{71.8}    \\ \midrule
            \multirow{4}{*}{{WebQSP}}  & Single Step            & 73.8                                                                     & 76.1                                                             & 75.7                 & 78.3            & 79.8           & 81.4           & 80.6            & 80.1            & 78.2             \\
                                       & Beam Search            & 77.1                                                                     & 78.9                                                             & 79.8                 & 82.6            & 82.5           & 85.4           & 81.7            & 86.6            & 81.8             \\
                                       & Lookahead              & 77.9                                                                     & 82.5                                                             & 82.1                 & 84.6            & 84.5           & 89.1           & 88.2            & 88.5            & 84.7             \\
                                       & \method                & \textbf{85.6}                                                            & \textbf{86.6}                                                    & \textbf{88.9}        & \textbf{90.4}   & \textbf{90.7}  & \textbf{93.8}  & \textbf{89.4}   & \textbf{93.9}   & \textbf{89.9}    \\ \midrule
            \multirow{4}{*}{{GrailQA}} & Single Step            & 71.7                                                                     & 76.6                                                             & 75.8                 & 79.1            & 73.1           & 79.3           & 77.4            & 79.2            & 76.5             \\
                                       & Beam Search            & 75.4                                                                     & 79.2                                                             & 77.7                 & 81.4            & 75.9           & 81.1           & 79.2            & 81.9            & 79.0             \\
                                       & Lookahead              & 74.6                                                                     & 77.4                                                             & 77.6                 & 80.6            & 83.9           & 86.9           & 84.2            & 87.9            & 81.6             \\
                                       & \method                & \textbf{80.9}                                                            & \textbf{84.9}                                                    & \textbf{84.7}        & \textbf{86.7}   & \textbf{89.6}  & \textbf{92.8}  & \textbf{90.8}   & \textbf{92.0}   & \textbf{87.8}    \\
            \bottomrule
        \end{tabular}
    }
\end{table*}

%% file: section/experiments.tex
\section{Experiments}

\subsection{Experimental Setup}

We conduct experiments on two categories of environments to evaluate long-horizon planning performance. The primary setting is KGQA (CWQ)\cite{talmor2018web}, WebQSP~\cite{yih2016value}, and GrailQA~\cite{gu2021beyond}), which fully satisfies our controlled diagnostic assumptions by formulating decision-making over explicit states with oracle-structured subgraphs that guarantee the existence of valid solution paths, thereby eliminating environment-side uncertainty. To evaluate generalization beyond structured graph traversal, we additionally consider a tool-use environment ALFWorld \citep{shridhar2021alfworld}, which involves long-horizon goal execution through sequences of discrete tool calls but does not strictly satisfy the controlled-state assumptions; it is therefore used solely to test cross-domain robustness. Within KGQA, we compare four planning strategies that differ only in action selection: single step, beam search, lookahead, and \method on top of ToG \citep{sun2024thinkongraph} and PoG \citep{chen2024plan}, alongside domain-specific KGQA baselines reported in prior work; in the tool-use setting, we implement the same planning strategies on top of ReAct \citep{yao2022react} and Reflexion \citep{shinn2023reflexion}.
Implementation details are provided in Appendix~\ref{sec:experimental}.

\subsection{Effectiveness}

\input{table/compare-to-sota}

\noindent\textbf{\method outperforms reasoning-based policies across models and frameworks.}
As shown in Table~\ref{tab:compare-planning}, \method consistently outperforms single step (reasoning-based policy), beam search (wider reasoning), and lookahead (shallow planning) across all datasets.
Accuracy under single step collapses as the horizon increases; beam search postpones but does not prevent this degradation.
Introducing limited lookahead improves performance, yet remains substantially weaker than full value propagation.
Importantly, these improvements are orthogonal to model scale.
\method with smaller backbones often matches or exceeds reasoning-based policies using substantially larger models, indicating that planning capability cannot be recovered by stronger reasoning alone.
The same conclusion holds across agent frameworks: \method yields consistent gains on both ToG and PoG, despite architectural differences.

\noindent\textbf{\method is competitive with or superior to specialized KGQA and MCTS-based approaches.}
As shown in Table~\ref{tab:compare-sota}, \method matches or surpasses prior MCTS and KGQA methods across all three benchmarks.
Notably, many existing approaches depend on task-specific training or carefully engineered search heuristics, whereas \method achieves comparable or better performance using a single, unified planning mechanism based on explicit lookahead and value propagation.
This suggests that principled planning mechanisms, rather than task-specific engineering, can serve as a general and effective foundation for graph-based reasoning and related structured decision-making problems.

\subsection{Planning Dynamics}

\input{table/diagnosis-summary}

\noindent\textbf{\method fundamentally improves long-horizon decision stability and recoverability.}
We evaluate decision behavior using three mechanism-level metrics (Table~\ref{tab:diagnosis-summary}): first-step trap selection, position of the first error, and recovery probability.
Single step exhibit systematic early myopia, selecting locally attractive but globally harmful actions and almost never recovering after an initial mistake.
Beam search extends this paradigm by increasing width, but remains driven by local scoring and thus commits prematurely to suboptimal prefixes.
Shallow lookahead mitigates these effects by incorporating limited future evaluation, but its benefits saturate quickly.
In contrast, \method consistently dominates all three metrics.
By propagating trajectory-level outcomes backward through the search tree, it suppresses early traps, postpones irreversible errors, and enables effective recovery.
This demonstrates that explicit planning mechanisms do not merely improve accuracy, but replace brittle reasoning-based decision making with stable and revisable long-horizon behavior.

\noindent\textbf{\method shifts the dominant failure mode from myopic commitment to exploration and termination limits.}
To clarify the scope of this improvement, we analyze residual failures under different decision mechanisms (Appendix~\ref{sec:failure-patterns}).
For single step and beam search, failures are dominated by premature commitment: actions are chosen for local plausibility rather than long-term feasibility, leading to irreversible deviation early in execution.
Limited lookahead reduces but does not eliminate this pathology.
\method directly removes this source of failure by evaluating actions through their induced future trajectories and propagating long-term outcomes backward to shape early decisions.
As a result, the dominant failure mode changes.
Under \method, errors primarily arise from insufficient exploration or imperfect termination control, such as looping or stopping before reaching a solution.
These failures are no longer driven by incorrect early commitments, but by constraints on search coverage and horizon management.
They therefore lie beyond what improved reasoning alone can resolve and point to distinct challenges for future planning-oriented agent design.

\subsection{Efficiency Analysis}

\begin{figure}[!t]
    \centering
    \includegraphics[width=\linewidth]{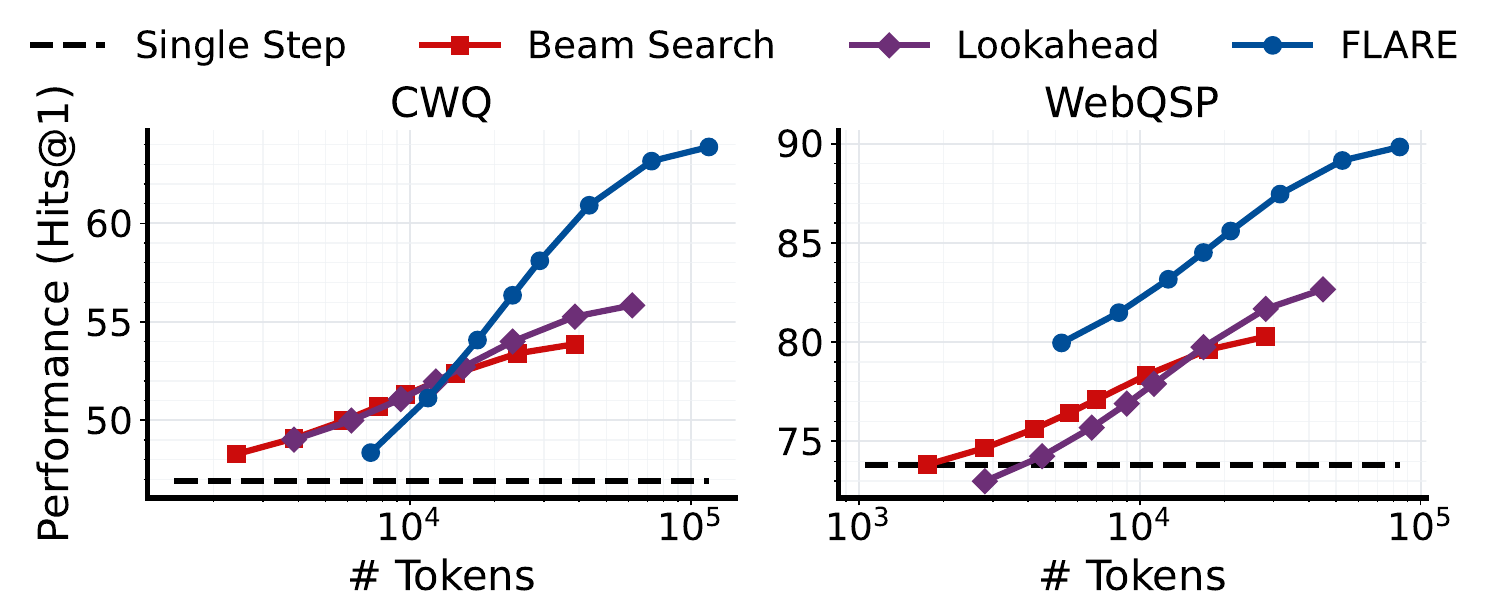}
    \caption{
        \textbf{Performance-budget trade-offs} of different planning strategies on CWQ and WebQSP.
        Results are plotted against total token budget, reflecting the inherent computational costs of different planning strategies.
        \method demonstrates superior scalability compared to beam search and lookahead baselines.
    }
    \label{fig:performance-budget-tradeoff}
\end{figure}

\noindent\textbf{\method converts additional computation into effective planning rather than brute-force search.}
Figure~\ref{fig:performance-budget-tradeoff} reports task performance as a function of total token budget on CWQ and WebQSP.
Single step incur minimal cost but show virtually no benefit from increased budget.
Beam search and lookahead initially improve with additional computation, but quickly saturate far below the performance of \method.
In contrast, \method exhibits sustained performance gains as the budget increases.
Although it incurs a higher cost per decision step, it achieves substantially higher accuracy at comparable or even lower total token budgets and continues to scale with additional computation.
This behavior indicates that \method allocates computation more productively, translating extra budget into improved long-horizon planning rather than into redundant local exploration.

\input{table/ablation-summary}

\noindent\textbf{Action pruning and trajectory memory provide complementary efficiency benefits.}
We ablate two efficiency-oriented components of \method, i.e., Action Pruning and Trajectory Memory, to isolate their roles under fixed computational budgets (Table~\ref{tab:ablation-summary}).
\textit{(1) Action Pruning:}
Removing Action Pruning causes a substantial performance drop at the default budget.
Recovering the full model's accuracy requires nearly three times more tokens, indicating that pruning improves efficiency by directing computation toward high-value actions rather than by merely lowering cost.
\textit{(2) Trajectory Memory:}
Removing Trajectory Memory markedly increases token usage with little change in accuracy.
Conversely, under matched budgets, performance deteriorates noticeably.
This pattern shows that trajectory memory improves efficiency by avoiding redundant evaluations of similar futures, thereby freeing budget for more informative exploration.


\subsection{Planning Behavior in Tool-Use Environment}

\noindent\textbf{\method exhibits distinct long-horizon planning behavior in tool-use environments.}
To evaluate whether our conclusions generalize beyond graph-based inference, we test \method in ALFWorld \citep{shridhar2021alfworld} via AgentGym \citep{xi2025agentgym}, where agents must complete long-horizon tasks by composing sequences of tool calls.
Despite different state and action representations, early decisions similarly induce irreversible failure modes.
As shown in Figure~\ref{fig:alfworld}, reasoning-based policies attain limited success and deviate from optimal trajectories early.
Strengthening planning mechanisms yields monotonic improvements.
\method achieves the highest success rates and substantially postpones the first error, demonstrating that explicit future evaluation and value propagation generalize beyond KGQA.
This provides cross-domain evidence that long-horizon failure arises from the decision paradigm itself, rather than from domain-specific structure.

\begin{figure}[!t]
    \centering
    \includegraphics[width=\linewidth]{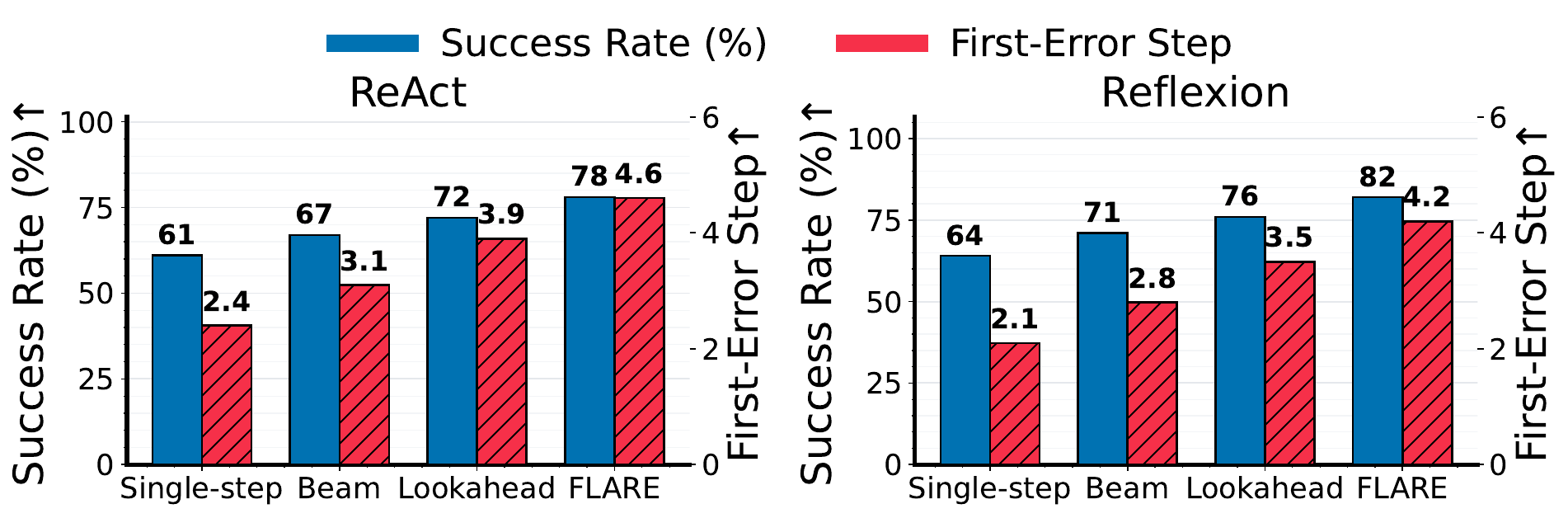}
    \caption{
        \textbf{Planning behavior on ALFWorld (tool-use).}
        We report the success rate and the first-error position for different planning strategies implemented on top of ReAct and Reflexion.
        \method consistently achieves higher success rates and higher first-error positions, demonstrating more stable long-horizon planning behavior beyond knowledge graph traversal.
    }
    \label{fig:alfworld}
\end{figure}

\input{table/compare-rl}

\noindent\textbf{\method clarifies the mechanistic distinction between online planning and reinforcement learning.}
Table~\ref{tab:compare-rl} categorizes agents by three planning components: explicit lookahead, backward value propagation, and commitment strategy.
Performance differences align with these mechanisms rather than with the training paradigm itself.
Prompt-based and RL-based agents, which lack explicit lookahead and value propagation at test time, remain vulnerable to early irreversible errors, even when augmented with trajectory sampling.
Planning-based methods improve robustness by explicitly evaluating future outcomes, and \method further benefits from receding-horizon commitment, enabling continual revision of early decisions.
Consequently, \method achieves the highest success rate on ALFWorld, reinforcing that coherent long-horizon behavior requires explicit planning mechanisms beyond local reasoning.

%% file: table/compare-to-sota.tex
\begin{table}[!t]
    \caption{
        \textbf{Comparison with state-of-the-art MCTS and KGQA methods.}
        \method achieves competitive or superior performance compared to prior approaches across all datasets.
    }
    \label{tab:compare-sota}
    \resizebox{\linewidth}{!}{
        \begin{tabular}{lccc}
            \toprule
            \textbf{Method}                     & \textbf{CWQ}  & \textbf{WebQSP} & \textbf{GrailQA} \\ \midrule
            IO Prompt \citep{brown2020language} & 37.6          & 63.3            & 29.4             \\
            CoT \citep{wei2022chain}            & 38.8          & 62.2            & 28.1             \\
            SC \citep{wang2023selfconsistency}  & 45.4          & 61.1            & 29.6             \\ \midrule
            ToT \citep{yao2023tree}             & 64.3          & 73.5            & 44.2             \\
            RAP \citep{hao2023reasoning}        & 68.2          & 75.6            & 49.7             \\
            LATS \citep{zhou2024language}       & 70.3          & 76.5            & 58.1             \\ \midrule
            KB-BINDER \citep{li2023few}         & -             & 74.4            & 58.5             \\
            ToG 2.0 \citep{ma2024think}         & -             & 81.1            & -                \\
            RoG \citep{luo2024reasoning}        & 62.6          & 85.7            & -                \\
            ToG \citep{sun2024thinkongraph}     & 67.6          & 82.6            & 81.4             \\
            PoG \citep{chen2024plan}            & 75.0          & 87.3            & \ul{84.7}        \\
            GCR \citep{luo2025graphconstrained} & \ul{75.8}     & \ul{92.2}       & -                \\
            DP \citep{ma2025deliberation}       & 74.6          & 87.5            & -                \\
            RwT \citep{shen2025reasoning}       & 72.4          & 87.0            & -                \\
            ProgRAG \citep{park2025prograg}     & 73.3          & 90.4            & -                \\
            DoG \citep{ma2025debate}            & 56.0          & 91.0            & 80.0             \\
            PathMind \citep{liu2025pathmind}    & 70.7          & 89.5            & -                \\
            \midrule
            \method (ToG)                       & 73.6          & 90.4            & 86.7             \\
            \method (PoG)                       & \textbf{78.8} & \textbf{93.9}   & \textbf{92.0}    \\
            \bottomrule
        \end{tabular}
    }
\end{table}

%% file: table/diagnosis-summary.tex
\begin{table}[!t]
    \centering
    \caption{
        \textbf{Mechanism-level comparison of long-horizon planning behavior.}
        We report metrics that explicitly quantify decision dynamics, including trap selection at the first step, the average position of the first error, recovery probability after the first error, and the correct decision prefixes.
        Aggregated results across datasets show that \method improves all aspects of planning behavior, supporting more stable long-horizon decision-making.
        The complete breakdown is provided in Table~\ref{tab:diagnosis-full}.
    }
    \label{tab:diagnosis-summary}
    \resizebox{\linewidth}{!}{
        \begin{tabular}{lccc}
            \toprule
            \multirow{2}{*}{\bf Method} & \textbf{Constructed}           & \textbf{First-Error}     & \textbf{Recovery@}              \\
                                        & \textbf{Trap@1 $\downarrow$} & \textbf{Step $\uparrow$} & \textbf{First-Error $\uparrow$} \\ \midrule
            Single Step                 & 55.6\%                       & 1.6                      & 5.4\%                           \\
            Beam Search                       & 71.9\%                       & 2.0                      & 11.4\%                          \\
            Lookahead                   & 23.6\%                       & 2.8                      & 22.4\%                          \\
            \method                       & \textbf{17.8\%}              & \textbf{3.2}             & \textbf{29.7\%}                 \\
            \bottomrule
        \end{tabular}
    }
    \vspace{-10pt}
\end{table}

%% file: table/ablation-summary.tex
\begin{table}[!t]
    \centering
    \caption{
        \textbf{Ablation of efficiency-oriented design choices in \method.}
        Performance-aligned and token-aligned variants isolate the computational and accuracy impact of removing action pruning and trajectory memory, respectively, demonstrating their complementary roles in improving planning efficiency. The full results are reported in Table \ref{tab:ablation-full}.
    }
    \label{tab:ablation-summary}
    \begin{tabular}{l | cc}
        \toprule
        \textbf{Methods}             & \textbf{Hits@1} & \textbf{\# Tokens}            \\
        \midrule
        \method (Full)               & 74.9            & 21k                           \\
        \midrule
        \multicolumn{3}{l}{\textbf{\textit{(a) Action Pruning (performance-aligned)}}} \\
        \midrule
        w/o Pruning                  & 69.5            & 15k                           \\
        w/o Pruning (Hits@1-Aligned) & 75.4            & 61k                           \\
        \midrule
        \multicolumn{3}{l}{\textbf{\textit{(b) Trajectory Memory (token-aligned)}}}    \\
        \midrule
        w/o Memory                   & 74.6            & 34k                           \\
        w/o Memory (Token-Aligned)   & 72.7            & 22k                           \\
        \bottomrule
    \end{tabular}
\end{table}

%% file: table/compare-rl.tex
\begin{table}[!t]
    \centering
    \caption{
        \textbf{Comparison of planning mechanisms and performance on ALFWorld.}
        We report success rates for prompt-based, RL-based, and planning-based agents.
        ``Part.'' denotes partial support and ``Rec.'' denotes receding-horizon commitment.
        \method is the only method that instantiates all three components and achieves the best performance.
    }
    \label{tab:compare-rl}
    \resizebox{\linewidth}{!}{
        \begin{tabular}{l| ccc | c}
            \toprule
            Method             & Lookahead & Value Prop. & Commit & Success (\%)     \\
            \midrule
            \multicolumn{5}{l}{\textit{\textbf{(a) Prompt-based Methods}}}           \\
            \midrule
            ReAct              & $\times$  & $\times$    & Rec.   & 61.3             \\
            ReAct + Beam       & $\times$  & $\times$    & Full   & 66.7             \\
            Reflexion          & $\times$  & $\times$    & Rec.   & 64.2             \\
            \midrule
            \multicolumn{5}{l}{\textit{\textbf{(b) Reinforcement Learning Methods}}} \\
            \midrule
            RL Policy (Greedy) & $\times$  & $\times$    & Rec.   & 69.4             \\
            RL Policy (Think)  & Part.     & Part.       & Full   & 77.4             \\
            \midrule
            \multicolumn{5}{l}{\textit{\textbf{(c) Planning-based Methods}}}         \\
            \midrule
            ReAct + MCTS       & Tree      & Back        & Full   & 73.5             \\
            ReAct + \method    & Tree      & Back        & Rec.   & \textbf{77.8}    \\
            \bottomrule
        \end{tabular}
    }
\end{table}

%% file: section/related_work.tex
\section{Related Work}


\noindent\textbf{Reasoning in LLM-based Agents.}
Multiple works have focused on improving the reasoning capability of LLM-based agents.
Chain-of-Thought \citep{wei2022chain} and self-consistency \citep{wang2023selfconsistency} improve multi-step reasoning via intermediate derivations and sampling.
Tree- and graph-structured variants, such as Tree-of-Thoughts \citep{yao2023tree,long2023large} and Graph-of-Thoughts \citep{besta2024graph}, enable branching and limited backtracking.
Agent frameworks including ReAct \citep{yao2022react}, Reflexion \citep{shinn2023reflexion}, and AutoGen \citep{wu2024autogen} integrate reasoning with environment interaction.
Despite their diversity, these approaches share a common paradigm: actions are selected step by step using local heuristics.


\noindent\textbf{Planning in LLM-based Agents.}
Recent work incorporates explicit search and planning into LLM agents.
In KGQA, Think-on-Graph \citep{sun2024thinkongraph} and Plan-on-Graph \citep{chen2024plan} formulate reasoning as graph traversal with beam search or plan revision.
Other approaches treat reasoning as search over structured state spaces \citep{ma2024think,markowitz2024tree,besta2024graph}.
Several methods further integrate LLMs with MCTS-style planning \citep{hao2023reasoning,zhang2024rest,long2025enhancing,wang2025damr,zhang2024rest,yoon2025monte}, introducing lookahead simulation and value backpropagation and achieving empirical gains on multi-step tasks.


\noindent\textbf{Offline vs. Online Planning.}
Another line of work trains agents with RL to optimize long-horizon objectives (offline planning) \citep{xi2025agentgymrl,guo2025deepseek,jin2025search,xu2025towards}, effectively amortizing planning into model parameters \citep{sutton1998reinforcement,wang2025reinforcement}.
At inference time, planning relies on autoregressive rollouts and sampling, rather than explicit revision of earlier actions.
In contrast, online planners (e.g., \method) explicitly evaluate and revise decisions during execution.


\noindent\textbf{Our Position.}
While recent studies report that LLM agents struggle with long-horizon tasks \citep{xi2025agentgymrl,erdogan2025plan,hu2025hiagent,gao2025beyond}, these limitations are typically treated as empirical observations.
We provide a planning-centric and theoretical account by modeling LLM reasoning as a step-wise greedy policy and analyzing its consequences.
This yields a mechanism-level explanation of error accumulation and establishes a formal distinction between reasoning and planning.

%% file: section/conclusion.tex


\section{Conclusion}

We show that long-horizon failures of LLM-based agents arise from a conceptual gap: step-by-step reasoning is not planning.
Even with explicit states and deterministic transitions, reasoning-based policies make irreversible early commitments due to their reliance on local evaluation.
To address this limitation, we propose \method, a minimal planning mechanism with explicit lookahead, backward value propagation, and limited commitment, enabling agents to evaluate future consequences and revise early decisions.
Experiments across tasks, agent frameworks, and model backbones demonstrate consistent improvements in long-horizon decision making.

\section*{Impact Statement}

This work highlights a fundamental distinction between reasoning and planning in LLM-based agents and argues that long-horizon reliability requires explicit planning mechanisms beyond step-wise reasoning.
By emphasizing future-aware evaluation and value propagation, our findings provide a principled foundation for building more reliable and interpretable agentic systems for multi-step decision problems \citep{liu2025advances}.
Beyond knowledge-intensive reasoning, this perspective applies broadly to domains requiring sustained decision making, including tool-augmented assistants \citep{qin2023toolllm,guo2024stabletoolbench}, workflow automation \citep{xiao2024flowbench}, program synthesis \citep{zhang2024codeagent}, and embodied systems \citep{duan2022survey}.

\noindent\textbf{Potential Benefits.}
A primary benefit of this work is improved reliability under long decision horizons \citep{yang2025mla}.
Explicit lookahead reduces irreversible early commitments, while value propagation allows downstream outcomes to systematically influence earlier choices.
These mechanisms yield decision processes that are more transparent than purely step-wise heuristics, facilitating debugging, auditing, and scientific analysis \citep{liu2025advances}.
As \method is a modular planning component, it can be integrated into existing agent frameworks without extensive task-specific engineering.
More broadly, the controlled diagnostic methodology introduced here may encourage more rigorous evaluation of agent planning capabilities by separating decision-mechanism failures from those caused by partial observability or unreliable feedback.

\noindent\textbf{Risks and Misuse Considerations.}
Stronger long-horizon planning also increases agents’ capacity to execute complex and sustained behaviors in high-impact domains.
Without appropriate safeguards \citep{xiang2024guardagent}, this may amplify risks of unintended or misaligned behavior \citep{liu2023trustworthy}, particularly when evaluation signals are imperfect proxies for human intent.
More capable planning may enable exploitation of reward-model weaknesses, pursuit of harmful long-term objectives, or unsafe tool chaining.
When external evaluators are used as planning signals, agents may additionally learn to exploit evaluator blind spots (reward hacking) \citep{gu2024survey}.
These risks highlight the need to accompany advances in planning with robust monitoring, sandboxing, and evaluation protocols.

\noindent\textbf{Limitations and Scope.}
Our empirical study focuses on controlled environments with deterministic transitions and structured action spaces, and does not address stochastic dynamics, partial observability, non-stationarity, or learning environment models from interaction \citep{busemeyer2010cognitive}.
Moreover, our approach assumes access to an evaluative signal at planning time, which may be noisy, costly, or unavailable in practice.
While KGQA provides a clean sandbox for isolating decision mechanisms, it does not establish performance in open-ended environments with less explicit state representations and feedback.

\noindent\textbf{Future Directions.}
Several directions are important for extending this work responsibly.
First, applying explicit planning to stochastic or partially observed environments \citep{xi2025agentgym} may require combining lookahead with belief-state inference \citep{friston2016active}, uncertainty-aware evaluation \citep{wagner2017path}, or learned world models \citep{hafner2025mastering}.
Second, integrating learning-based value estimation with explicit planning \citep{schrittwieser2020mastering} may amortize long-horizon reasoning while preserving future-aware decision making.
Third, deploying planning-based agents in open-ended settings calls for embedding safety constraints into planning objectives, strengthening evaluator robustness, and developing benchmarks for goal misgeneralization and safe tool use \citep{liu2023trustworthy}.
We encourage future research to jointly advance planning capability and safety to ensure that increased autonomy is matched by stronger alignment and oversight.

%% file: appendix/related_work.tex
\section{Additional Related Works}
\label{sec:additional_related_work}

\subsection{Classical Planning and MCTS}

Classical planning tasks can be generally formulated as the generation of action sequences that can transforms an initial state into a desired goal state. Foundational research in classical planning like STRIPS~\cite{fikes1971strips} and PDDL~\cite{aeronautiques1998pddl} assumes fully observable and deterministic environments with pre-defined state transition functions and explicit goal conditions. Following this setting, classical planning is typically formulated as a graph search problem~\cite{blum1997fast}. Early approaches mostly rely on systematic search algorithms, such as forward and backward chaining~\cite{pednault1989adl,ghallab2004automated}. Recently, informative heuristics like relaxed planning graphs and landmarks are integrated into planners to reduce search space and enable scalable planning~\cite{richter2010lama,hoffmann2001ff,bonet2001planning}.

Despite these advances, classical planning methods often struggle with extremely large state space, sparse reward and long-term planning~\cite{haslum2019introduction}. These challenges motivated the development of alternative search paradigms, while MCTS has emerged as a powerful alternative option. MCTS is originally developed for decision-making in game-playing~\cite{coulom2006efficient}. Instead of relying on explicit heuristics, MCTS incrementally builds a search tree using randomized rollouts and propagates value estimates backward to guide future exploration, which is usually bounded by Upper Confidence Bound (UCB) to balance exploration and exploitation. These domain-independent properties facilitate the application of MCTS to classical planning, especially the planning tasks in unseen domains~\cite{swiechowski2023monte}. MCTS has been successfully applied to domain-specific planning problems by incorporating hierarchical planner~\cite{shao2021107067} and domain knowledge~\cite{eiffert2020path,dam2022monte,hennes2015trajectory,zhao2025controllable}. More recently, to further implement domain-independent planning, researchers have attempted to combine MCTS with multi-armed bandit algorithm~\cite{asai2024extreme,wissow2024scale} for scalable long-term planning.
These empirical success of MCTS in classical planning has demonstrated the potential of MCTS in improving LLMs' reasoning ability in unseen environments.

\subsection{Diagnosis of LLM-based Agent}


The diagnostic study of LLM-based agents \citep{ma2025autodata,zhang2025agentrouter,li2026longda,thapaliya2025semantic} has evolved from simple accuracy metrics to a multi-dimensional investigation of internal behavioral pathologies. Early research in this domain was primarily concerned with the veracity of generated content, focusing on the taxonomy and mitigation of hallucinations~\cite{ji2023survey,bang2023multitask}. As the field transitioned toward complex problem-solving, the diagnostic lens shifted to reasoning failures. While Chain-of-Thought initially appeared to unlock latent logic~\cite{wei2022chain}, subsequent diagnostics revealed significant compositional gaps, where models struggle to integrate multiple simple reasoning steps into a coherent proof~\cite{press2023measuring}. This brittleness was further highlighted by the discovery of the Reversal Curse~\cite{berglund2023reversal}, which demonstrated that LLMs fail to generalize basic logical identities, exposing a fundamental lack of bidirectional understanding.

Beyond logical errors, recent literature has identified shortcut reasoning (or lazy reasoning) as a pervasive behavioral bias. Diagnostics have shown that LLMs often bypass rigorous deduction by exploiting statistical artifacts in the prompt, such as positional biases which is exemplified by the Lost in the Middle phenomenon \cite{liu2024lost,zhang2024found}, or the tendency to rely on label frequency rather than context \cite{reif2024beyond,mckenna2023sources}. \citet{dziri2023faith} argued that many observed successes in multi-step reasoning are not the result of systematic state-space search but are instead lucky pattern matches. This skepticism was reinforced later by studies on self-correction, where researchers found that LLMs often struggle to self-diagnose and fix their own errors without external feedback, often deteriorating the solution quality through hallucinated corrections \cite{huang2023large,pan2023automatically,tsui2025self}.

Recently researches has begun to explore System 2 planning and reinforcement learning through search-based architectures~\cite{yoon2025monte,li2025system}, such as DeepSeek-R1 \cite{guo2025deepseek} and Search-R1 \cite{jin2025search}. Most recently, studies have started diagnosing Myopic Planning, the tendency of models to prioritize immediate tokens over long-term goal feasibility \cite{ma2024non,laban2025llms,xi2025agentgymrl}, and the degree to which LLMs possess internal world models to anticipate the consequences of their actions \cite{hu2023language,guo2025sample,levy2025worldllm}. 

In contrast to these lines of work, we focus on a different object of analysis: the agent's decision process over action sequences.
Rather than studying whether intermediate reasoning steps are logically valid or faithful, we ask whether the induced action-selection mechanism can support coherent behavior over long horizons.
We adopt a planning-centric perspective that formalizes LLM reasoning as a step-wise greedy policy, and analyze its consequences at the level of trajectories, commitments, and recoverability, allowing us to theoretically distinguish reasoning from planning capabilities.

%% file: appendix/theory.tex
\section{Suboptimality of Truncated Lookahead}
\label{sec:suboptimality}

While Proposition~\ref{prop:lookahead} establishes the strict superiority of $k \ge 1$ lookahead over step-wise greedy and beam search policies,
practical LLM agents operate under finite computational budgets, resulting in an effective lookahead depth $k$ that is often much smaller than the episode horizon $H$.
We now characterize the worst-case suboptimality of a truncated $k$-lookahead policy when critical rewards are systematically delayed beyond the lookahead horizon. We provide the proof in Appendix \ref{sec:proof}.

\paragraph{The $k$-Lookahead Policy.}
For a given lookahead depth $k \in \{1,\dots,H-1\}$, define the truncated lookahead value
\begin{equation}
    \tilde{Q}_k(s,a)
    =
    \max_{\pi'}
    \sum_{t=0}^{k-1} \hat{r}(s_t,a_t,s_{t+1}),
\end{equation}
where $s_0=s$, $a_0=a$, $s_{t+1}=T(s_t,a_t)$, and $\pi'$ denotes a policy over the subsequent $k-1$ steps.
All rewards beyond depth $k$ are treated as zero.
The $k$-lookahead policy $\pi_k$ selects actions according to
\begin{equation}
    \pi_k(s) \in \arg\max_{a \in \mathcal{A}(s)} \tilde{Q}_k(s,a),
\end{equation}
with ties broken by a fixed step-wise surrogate score $\hat{u}(s,a)$.

\begin{proposition}[Worst-Case Suboptimality of $k$-Lookahead]
    \label{prop:suboptimality}
    Let $H$ be the episode horizon and let $k \in \{1,\dots,H-1\}$.
    Let
    \begin{equation}
        R_{\max} = \max_{s,a,s'} |\hat{r}(s,a,s')|
    \end{equation}
    denote the maximum magnitude of any single-step reward.
    Then there exists a class of deterministic environments $\mathcal{G}$ such that
    the additive suboptimality gap of the $k$-lookahead policy $\pi_k$ satisfies
    \begin{equation}
        R(\tau^*) - R(\tau(\pi_k))
        \;\ge\;
        R_{\max} \cdot \left\lfloor \frac{H-1}{k+1} \right\rfloor.
    \end{equation}
    Consequently, in the worst case, the performance gap of truncated lookahead grows linearly with the ratio $H/(k+1)$.
\end{proposition}

Proposition~\ref{prop:suboptimality} establishes that the failure mode observed in Proposition~\ref{prop:greedy} can be generalized to arbitrarily deep lookahead policies, provided the crucial rewards are placed just beyond the lookahead horizon $k$. The bound demonstrates that the gap between the true optimal return and the return achieved by $\pi_k$ grows linearly with the ratio of the total horizon to the lookahead depth.

\section{Proofs}
\label{sec:proof}

\subsection{Proof of Proposition \ref{prop:greedy}}

\begin{proof}
    Fix any $H \ge 2$ and $M > 0$.
    We construct a deterministic environment as follows.

    Let the initial state be $s_0 := s$.
    Define two successor states $s_a$ and $s_b$ such that
    \begin{equation}
        T(s_0,a) = s_a,
        \qquad
        T(s_0,b) = s_b.
    \end{equation}
    At both $s_a$ and $s_b$, only a single action $a_\bot$ is available, which transitions to an absorbing terminal state $s_\bot$.
    The horizon is padded by remaining in $s_\bot$ with zero reward to ensure total length $H$.

    Define the reward function as
    \begin{align}
        \hat{r}(s_0,a,s_a)         & = 0, \\
        \hat{r}(s_0,b,s_b)         & = 0, \\
        \hat{r}(s_a,a_\bot,s_\bot) & = 0, \\
        \hat{r}(s_b,a_\bot,s_\bot) & = M,
    \end{align}
    and let all other rewards be zero.

    Since $\hat{u}(s_0,a) > \hat{u}(s_0,b)$, the step-wise greedy policy $\pi_{\text{greedy}}$ selects action $a$ at $s_0$, yielding
    \(
    R(\tau(\pi_{\text{greedy}})) = 0.
    \)
    In contrast, the policy that selects $b$ at $s_0$ achieves return $M$.
    As $M$ is arbitrary, the suboptimality gap can be made arbitrarily large.
\end{proof}

\subsection{Proof of Proposition \ref{prop:beam}}

\begin{proof}
    Fix $B$, $H \ge 2$, and $M > 0$.
    Let the initial state be $s_0$.
    Define $B+1$ decoy actions $a_1,\dots,a_{B+1}$ and one good action $g$ such that
    \begin{equation}
        T(s_0,g) = s_g,
        \qquad
        T(s_0,a_j) = s_j \quad (j=1,\dots,B+1).
    \end{equation}

    At each $s_j$, only a terminal action $a_\bot$ is available, yielding zero reward.
    At $s_g$, the same terminal action yields reward $M$.
    Formally,
    \begin{equation}
        \hat{r}(s_j,a_\bot,s_\bot) = 0,
        \qquad
        \hat{r}(s_g,a_\bot,s_\bot) = M.
    \end{equation}

    Define the surrogate scores as
    \begin{equation}
        \hat{u}(s_0,a_j) = 1 \quad (j=1,\dots,B+1),
        \qquad
        \hat{u}(s_0,g) = 0,
    \end{equation}
    and $\hat{u}=0$ elsewhere.

    At depth $1$, the prefix corresponding to $g$ has accumulated score $S_1=0$,
    while each decoy prefix has score $S_1=1$.
    Since there are $B+1$ decoy prefixes with strictly higher scores,
    any beam-search policy retaining only the top $B$ prefixes must prune the optimal prefix.
    Once pruned, the optimal trajectory cannot be recovered, and the final return is zero.

    As the optimal return $M$ can be arbitrarily large, the result follows.
\end{proof}

\subsection{Proof of Proposition \ref{prop:lookahead}}

\begin{proof}
    We reuse the environment construction from Proposition~\ref{prop:greedy}.
    At the initial state $s_0$, there exist two actions $a$ and $b$ such that
    $\hat{u}(s_0,a) > \hat{u}(s_0,b)$.
    Action $a$ leads deterministically to a successor state $s_a$ with zero future reward,
    while action $b$ leads to $s_b$ from which a reward $M > 0$ is obtained at the next step.

    A step-wise policy that selects actions based solely on $\hat{u}$ necessarily chooses $a$
    and receives total return $0$.
    The same holds for any finite-width beam search ranked only by accumulated surrogate scores,
    as shown in Proposition~\ref{prop:beam}.

    Now consider a lookahead-based policy with horizon $k=1$.
    At state $s_0$, the policy evaluates each candidate action by simulating one-step rollouts:
    \begin{align}
        \text{Value}(a) & = \hat{r}(s_0,a,s_a) + \hat{r}(s_a,a_\bot,s_\bot) = 0, \\
        \text{Value}(b) & = \hat{r}(s_0,b,s_b) + \hat{r}(s_b,a_\bot,s_\bot) = M.
    \end{align}
    Since $M > 0$, the lookahead policy correctly selects action $b$ at $s_0$,
    thereby achieving the optimal return $M$.

    As $M$ can be arbitrarily large, the separation between step-wise decision-making and lookahead planning is strict.
\end{proof}

\subsection{Proof of Proposition \ref{prop:suboptimality}}

\begin{proof}
    Fix $H$, $k \in \{1,\dots,H-1\}$, and $R_{\max} > 0$.
    We construct a deterministic environment in which all non-zero rewards are placed strictly beyond the lookahead horizon.

    Let
    \begin{equation}
        N = \left\lfloor \frac{H-1}{k+1} \right\rfloor.
    \end{equation}
    The state space contains a sequence of states
    $\{s^{(0)}, s^{(1)}, \dots, s^{(N)}\}$,
    where $s^{(0)}$ is the initial state and $s^{(N)}$ is terminal.
    For each $i \in \{0,\dots,N-1\}$, the agent at state $s^{(i)}$ can choose between two actions:
    a good action $g$ and a decoy action $d$.

    From state $s^{(i)}$, both actions deterministically induce a chain of $k$ transitions with zero reward.
    Specifically,
    \begin{equation}
        \hat{r}(s^{(i)}, g, s_{g,1}) = \hat{r}(s^{(i)}, d, s_{d,1}) = 0,
    \end{equation}
    and for all $j=1,\dots,k-1$,
    \begin{equation}
        \hat{r}(s_{g,j}, a, s_{g,j+1}) = \hat{r}(s_{d,j}, a, s_{d,j+1}) = 0.
    \end{equation}
    At the $(k+1)$-th transition, which lies strictly outside the lookahead horizon,
    the two paths diverge:
    the decoy path yields zero reward and transitions to $s^{(i+1)}$,
    while the good path yields reward $R_{\max}$ and also transitions to $s^{(i+1)}$.

    We define the step-wise surrogate function such that
    \begin{equation}
        \hat{u}(s^{(i)}, d) > \hat{u}(s^{(i)}, g),
    \end{equation}
    and $\hat{u}=0$ elsewhere.
    Consequently, when $\tilde{Q}_k(s^{(i)}, g) = \tilde{Q}_k(s^{(i)}, d) = 0$,
    ties are consistently broken in favor of the decoy action.

    Under this construction, at every state $s^{(i)}$, the $k$-lookahead policy $\pi_k$
    evaluates both actions as having equal truncated value,
    \begin{equation}
        \tilde{Q}_k(s^{(i)}, g) = \tilde{Q}_k(s^{(i)}, d) = 0,
    \end{equation}
    and therefore selects the decoy action $d$.
    As a result, $\pi_k$ receives zero reward throughout the episode:
    \begin{equation}
        R(\tau(\pi_k)) = 0.
    \end{equation}

    In contrast, the optimal policy $\pi^*$ selects the good action $g$ at every state $s^{(i)}$.
    At each segment, it collects reward $R_{\max}$ at the $(k+1)$-th step,
    yielding a total return
    \begin{equation}
        R(\tau^*) = N \cdot R_{\max}.
    \end{equation}
    Therefore,
    \begin{equation}
        R(\tau^*) - R(\tau(\pi_k))
        =
        R_{\max} \cdot \left\lfloor \frac{H-1}{k+1} \right\rfloor,
    \end{equation}
    which establishes the claimed lower bound.
\end{proof}

%% file: appendix/method.tex
\section{\method Details}
\label{sec:technical details}

\subsection{Framework Illustration}
\label{sec:framework}

We illustrate the \method framework in Figure \ref{fig:framework}.

\begin{figure}[!hb]
    \centering
    \includegraphics[width=0.95\linewidth]{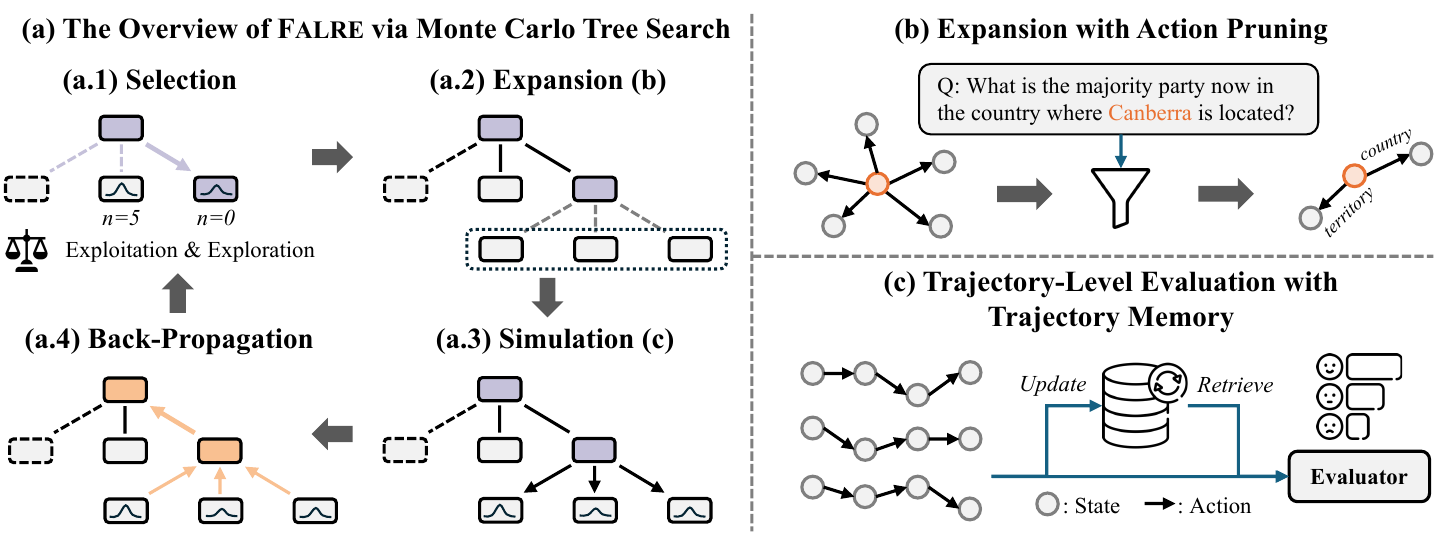}
    \caption{
    \textbf{Overview of \method via Monte Carlo Tree Search.}
    (\textit{a}) Planning proceeds in a receding-horizon manner. At each step, \method expands candidate actions with pruning for efficiency, performs repeated simulations, evaluates trajectories using trajectory-level feedback, propagates returns backward to update action values, and executes the highest-valued action before replanning in the next state.
    (\textit{b}) During simulation, actions are selected using a UCB-style rule to balance exploration and exploitation.
    (\textit{c}) Trajectory-level evaluation is amortized with a bounded trajectory memory, which retrieves and reuses scores for similar trajectories to reduce evaluation cost.
    Overall, the design enforces explicit lookahead and value propagation, allowing future outcomes to shape early decisions.
    }
    \label{fig:framework}
\end{figure}

\subsection{Pseudo-code}
\label{sec:pseudocode}

We summarize the comprehensive algorithmic procedure of FLARE in Algorithm \ref{alg:flare}.

\input{algorithm/method}

\subsection{Detailed Model Implementation}

\subsubsection{Explicit Lookahead}

A central failure mode in our diagnosis is that step-wise decision-making evaluates actions in isolation, based solely on local surrogate scores.
Explicit lookahead addresses this limitation by evaluating each candidate action through the future trajectories it induces under known environment dynamics.
Rather than assigning scores to actions directly, \method estimates their long-term consequences by simulating possible futures and comparing the resulting trajectory-level outcomes.

\noindent\textbf{Selection.}
To support explicit lookahead, \method organizes simulated futures using a search tree rooted at the current decision state, where nodes correspond to states and edges to actions.
Each simulation traverses this tree starting from the root $s_0 = s$, producing a partial trajectory
\(
\tau = (s_0, a_0, s_1, \dots).
\)

Selection determines how simulation effort is allocated across existing branches of the tree.
It does not decide which action will be executed, but balances exploration of under-sampled futures with refinement of futures whose downstream outcomes appear informative for long-horizon evaluation.
Importantly, selection does not rely on step-wise surrogate rewards and does not directly commit to executable actions; it only controls which futures are explored to support trajectory-level evaluation.

To guide this allocation, \method employs a UCB-style rule,
\begin{equation}
    \label{eq:select}
    a_t = \arg\max_{a \in \mathcal{A}(s_t)}
    \left(
    Q(s_t,a) + c \sqrt{\frac{\log N(s_t)}{N(s_t,a) + 1}}
    \right),
\end{equation}
where $Q(s_t,a)$ aggregates trajectory-level evaluations from previous simulations and $N(s_t,a)$ denotes visit counts.
Crucially, selection is guided by future trajectory outcomes rather than step-wise rewards.

\noindent\textbf{Expansion with Action Pruning.}
While selection governs traversal within the existing search tree, expansion determines which new future continuations are introduced for evaluation.
When a simulation encounters a previously uninstantiated state--action pair, \method expands the tree by adding the corresponding successor state, thereby exposing new candidate futures for trajectory-level comparison.
Expansion is performed incrementally and does not attempt to enumerate all possible continuations.

To keep explicit lookahead tractable, \method introduces action pruning as a key efficiency component.
At each decision state $s$, an action proposal function
\(
\phi : \mathcal{S} \rightarrow 2^{\mathcal{A}}
\)
constructs a restricted candidate set
\begin{equation}
    \mathcal{A}_k(s) = \phi(s), \qquad |\mathcal{A}_k(s)| \le k,
\end{equation}
which defines the actions whose induced futures may be introduced during expansion.
In our instantiation, $\phi$ is implemented using a LLM, though the planning mechanism is agnostic to the specific form of $\phi$.

Action pruning does not assign values to actions, does not determine their preference, and does not introduce additional decision signals, heuristics, or reward shaping.
It solely delimits the scope of explicit lookahead for computational efficiency.

\subsubsection{Value Propagation}

Explicit lookahead exposes candidate future trajectories, but effective long-horizon planning additionally requires that information from these futures influence early decisions.
\method achieves this through value propagation, which aggregates trajectory-level evaluations and propagates their outcomes backward to revise the estimated quality of earlier actions.

This mechanism is specifically designed to counteract the irreversibility of early commitments induced by step-wise greedy decision-making, as formalized in Proposition~3.1 and Proposition~3.2.

\noindent\textbf{Trajectory-Level Evaluation with Trajectory Memory.}
To support value propagation from future outcomes, \method evaluates simulated futures at the trajectory level rather than scoring individual actions or intermediate states.
When a simulated trajectory
\(
\tau = (s_0, a_0, \dots, s_H)
\)
reaches a terminal condition or a predefined horizon $H$, it is assigned a cumulative return
\begin{equation}
    \label{eq:return}
    R(\tau) = \sum_{t=0}^{H-1} \hat{r}(s_t, a_t, s_{t+1}),
\end{equation}
where $\hat{r}$ denotes an evaluative signal available at planning time.
This definition is agnostic to the specific form of feedback.
Trajectory-level evaluation ensures that actions are judged solely by their downstream consequences, avoiding biases induced by step-wise surrogate signals.

To amortize trajectory evaluation under a fixed computational budget, \method maintains a trajectory memory
\(
\mathcal{M} = \{(\tilde{\tau}, R(\tilde{\tau}))\}
\)
of previously evaluated trajectories.
Reuse is governed by a similarity function
\(
\mathrm{sim} : \mathcal{T} \times \mathcal{T} \rightarrow \mathbb{R}.
\)
Given a newly simulated trajectory $\tau$, \method retrieves
\begin{equation}
    \label{eq:retrieve}
    \tilde{\tau}^* = \arg\max_{\tilde{\tau} \in \mathcal{M}} \mathrm{sim}(\tau, \tilde{\tau}),
\end{equation}
and reuses $R(\tilde{\tau}^*)$ as an estimate for $R(\tau)$ whenever $\mathrm{sim}(\tau, \tilde{\tau}^*) \ge \delta$.
Otherwise, $R(\tau)$ is obtained by direct evaluation.
The memory is bounded by $|\mathcal{M}| \le M$ to control storage cost.

Trajectory memory does not modify the evaluation objective, alter trajectory returns, or bias action selection beyond reusing previously computed feedback.
Its role is limited to amortizing trajectory-level evaluation under fixed computational budgets.

\noindent\textbf{BackPropagation.}
For each simulated trajectory $\tau = (s_0, a_0, \dots, s_H)$ with return $R(\tau)$, \method propagates the trajectory-level outcome backward to all state--action pairs $(s_t, a_t)$ along the trajectory.
The value estimate $Q(s_t, a_t)$ is updated by aggregating returns from all trajectories passing through that pair, allowing future outcomes to revise the estimated quality of earlier actions.

After a fixed number of simulations, \method selects the next action by reading out the propagated values at the root,
\begin{equation}
    a^* = \arg\max_{a \in \mathcal{A}(s)} Q(s,a).
\end{equation}
Action selection is thus governed by aggregated trajectory-level evidence rather than step-wise heuristics, enabling principled revision of early decisions based on their downstream consequences.

\subsection{Time Complexity Analysis}
\label{sec:complexity}

We analyze the time complexity of Algorithm~\ref{alg:flare} per decision step with simulation budget $S$ and planning horizon $H$.

\noindent\textbf{Notation.}
Let $k$ be the maximum number of actions returned by the proposal function $\phi$ (i.e., $|\mathcal{A}_k(s)|\le k$). Let $C_T$ denote the cost of one transition call $T(s,a)$, $C_\phi$ the cost of one invocation of $\phi(s)$, $C_r$ the cost of computing $R(\tau)$ in Eq.~(\ref{eq:return}), and $C_{\mathrm{sim}}$ the cost of computing $\mathrm{sim}(\tau,\tilde{\tau})$ for a pair of trajectories. Let $|\mathcal{M}| \le M$ be the trajectory-memory size. We write $S_{\mathrm{exp}}$ for the number of expansions performed across $S$ simulations; clearly $S_{\mathrm{exp}} \le S$.

\noindent\textbf{Per-simulation cost.}
Each simulation constructs a trajectory of length at most $H$, but may terminate earlier due to the \texttt{break} upon expanding an unexpanded state. We upper bound costs using $H$.

\emph{(1) Tree traversal and transitions.}
In the inner loop, each step performs selection (Eq.~(\ref{eq:select})) over at most $k$ actions and one transition call $T(s,a)$, yielding
\begin{equation}
    O(Hk + H C_T).
\end{equation}

\emph{(2) Expansion with action pruning.}
When a new state is expanded, Algorithm~\ref{alg:flare} calls $\phi(s)$ once and materializes at most $k$ outgoing actions, costing
\begin{equation}
    O(C_\phi + k)
\end{equation}
for that simulation. If no new state is expanded, this cost is $0$.

\emph{(3) Trajectory evaluation and memory retrieval.}
Memory retrieval requires identifying $\tilde{\tau}^*$ via Eq.~(\ref{eq:retrieve}), which in the worst case scans up to $M$ stored trajectories and computes similarity for each:
\begin{equation}
    O(M C_{\mathrm{sim}}).
\end{equation}
If reuse fails, direct evaluation additionally costs $O(C_r)$.

\emph{(4) Backpropagation.}
The backpropagation loop visits each state-action pair along $\tau$, costing
\begin{equation}
    O(H),
\end{equation}
assuming $Q(s,a)$ is updated using an incremental aggregate (e.g., running mean) rather than recomputing expectations from scratch.\footnote{If $Q(s,a)$ is recomputed by scanning all trajectories in $\mathcal{M}$, the update cost would add a factor depending on $|\mathcal{M}|$; our implementation maintains per-edge aggregates for $O(1)$ updates.}

\noindent\textbf{Total complexity.}
Summing over $S$ simulations, the total time complexity is
\begin{align}
    O\Big(
     & S (Hk + H C_T + M C_{\mathrm{sim}} + C_r + H)
    \;+\; S_{\mathrm{exp}} (C_\phi + k)
    \Big),
\end{align}
where $S_{\mathrm{exp}} \le S$ is the number of expansions (and thus the number of $\phi$ calls). Under fixed $S,H,k,M$, \method has a bounded per-step compute budget, and the dominant terms scale linearly with $S$ and $H$.

\noindent\textbf{Remarks.}
Compared to step-wise methods that invoke $\phi$ once per decision and select actions greedily or via beam search, \method trades additional computation $O(SH)$ for explicit lookahead and value propagation, enabling future-aware evaluation of early actions.

%% file: algorithm/method.tex
\begin{algorithm}[!ht]
    \caption{FLARE}
    \label{alg:flare}
    \begin{algorithmic}[1]
        \REQUIRE Initial state $s_0$, transition function $T$, action proposal function $\phi$, planning horizon $H$, simulation budget $S$
        \ENSURE Optimal action $a^*$

        \STATE Initialize $N(s,a) \leftarrow 0, Q(s,a) \leftarrow 0, \mathcal{M} \leftarrow \emptyset$

        \FOR{$i = 1$ to $S$}
        \STATE $s \leftarrow s_0, \tau \leftarrow (s_0)$

        \STATE // Explicit Lookahead
        \FOR{$t = 1$ to $H$}
        \IF{$s$ is not expanded}
        \STATE $\mathcal{A}_k(s) \leftarrow \phi(s)$
        \STATE Expand $s$ with $\mathcal{A}_k(s)$
        \STATE \textbf{break}
        \ENDIF

        \STATE Select $a$ via Eq. (\ref{eq:select})

        \STATE $s' \leftarrow T(s,a)$
        \STATE $\tau \leftarrow \tau \cup (a, s')$
        \STATE $s \leftarrow s'$
        \ENDFOR

        \STATE // Trajectory Evaluation
        \STATE Retrieve $\tilde{\tau}^*$ from $\mathcal{M}$ via Eq. (\ref{eq:retrieve})
        \IF{$\mathrm{sim}(\tau,\tilde{\tau}^*) \ge \delta$}
        \STATE $R(\tau) \leftarrow R(\tilde{\tau}^*)$
        \ELSE
        \STATE Evaluate $R(\tau)$ via Eq. (\ref{eq:return})
        \STATE $\mathcal{M} \leftarrow \mathcal{M}\cup (\tau, R(\tau))$
        \ENDIF

        \STATE // Backpropagation
        \FORALL{$(s, a) \in \tau$}
        \STATE $N(s,a) \leftarrow N(s,a) + 1$
        \STATE $Q(s,a) \leftarrow \mathbb{E}[R(\tau)|(\tau, R(\tau))\in \mathcal{M}, (s,a)\subset \tau]$
        \ENDFOR
        \ENDFOR

        \STATE $a^* \leftarrow \arg\max_{a\in \mathcal{A}_k(s_0)} Q(s_0,a)$
        \STATE \textbf{return} $a^*$
    \end{algorithmic}
\end{algorithm}

%% file: appendix/experimental_setup.tex
\section{Experimental Details}
\label{sec:experimental}

\subsection{Benchmarks}

\noindent\textbf{Knowledge Graph Question Answering.}
We instantiate long-horizon planning in a controlled and fully structured setting using KGQA, which is an extension of graph learning \citep{wang2024gft,wang2025generative,wang2025towards,wangbeyond} in question answering.
Following \citep{sun2024thinkongraph}, we evaluate all methods on three widely-used KGQA benchmarks:
ComplexWebQuestions (CWQ)~\cite{talmor2018web}, WebQSP~\cite{yih2016value} and GrailQA~\cite{gu2021beyond},
all constructed on Freebase \citep{bollacker2008freebase} and requiring multi-hop relational reasoning to traverse from a topic entity to the target answer entity.
These datasets exhibit substantially different shortest-path length distributions, ranging from relatively shallow queries in WebQSP to long-horizon and compositional queries in CWQ and GrailQA.
WebQSP is designed to test I.I.D. generalization on questions over WebQuestions, which mainly contains questions that can be answered with 1-2 hop reasoning paths. CWQ is extended from WebQSP with four additional types of complex questions, involving moderately long multi-hop reasoning (typically 2-4 hops). GrailQA is a diverse KGQA dataset designed to test different levels of generalization, which emphasizes more complex compositional reasoning.
This diversity makes KGQA a suitable sandbox for systematically studying how planning performance degrades as the required decision horizon increases.
Crucially, we treat KGQA not as a question answering task per se, but as a sequential decision-making problem defined on an explicit state space.
To isolate agent-side planning behavior from confounding factors such as retrieval failure or incomplete environment access, we adopt an oracle-structure evaluation setting throughout all experiments.
For each question, we construct a subgraph that is guaranteed to contain at least one valid path from the topic entity to the ground-truth answer.
Under this setting, the agent has explicit access to the state space, transition structure, and successor actions at each step.
As a result, any performance difference can be attributed solely to the planning strategy rather than missing environment information or knowledge access.

\noindent\textbf{Tool Usage.}
We additionally evaluate planning performance in ALFWorld \citep{shridhar2021alfworld}, a text-based household environment aligned with the embodied ALFRED benchmark~\citep{shridhar2020alfred}.
ALFWorld consists of six categories of goal-oriented tasks in which an agent must accomplish a high-level objective by navigating and interacting with a simulated home through discrete text actions.
Individual task instances often require long action sequences and structured exploration over dozens of locations, making the environment a representative testbed for long-horizon decision making with hierarchical subgoals.
Beyond horizon length, ALFWorld also requires agents to leverage commonsense knowledge about typical object–location relationships (e.g., lamps are more likely to be found on desks or shelves), introducing an additional source of difficulty that is absent in fully structured graph environments.
Following prior work, we evaluate on the standard set of unseen test games under a task-specific setup.
To ensure a fair comparison between different agent frameworks, all methods are evaluated under identical environment access and action interfaces, and differ only in their planning strategies.

\subsection{Planning Strategies}

We compare three representative decision paradigms that differ in how future outcomes are incorporated into action selection.

\noindent\textbf{Single-Step Greedy Policy.}
The single-step baseline corresponds to a reasoning-based decision policy.
At each state $s_t$, the agent selects the action that maximizes a local step-wise surrogate score,
\begin{equation}
    a_t^{\text{greedy}} = \arg\max_{a \in \mathcal{A}(s_t)} \hat{r}(s_t, a, T(s_t,a)),
\end{equation}
where $\hat{r}$ denotes the evaluative signal available at planning time and $T$ is the deterministic transition function.
This policy optimizes only immediate outcomes and does not account for downstream consequences.
It therefore instantiates the step-wise greedy decision mechanism.

\noindent\textbf{Beam Search.}
Beam search extends the single-step paradigm by maintaining multiple candidate partial trajectories.
At depth $t$, the algorithm keeps a beam $\mathcal{B}_t$ consisting of the top-$B$ prefixes ranked by accumulated step-wise scores,
\begin{equation}
    S(\tau_{0:t}) = \sum_{i=0}^{t-1} \hat{r}(s_i, a_i, s_{i+1}),
\end{equation}
where $\tau_{0:t} = (s_0,a_0,\dots,s_t)$ denotes a partial trajectory.
At each step, all prefixes in $\mathcal{B}_t$ are expanded by one action, and only the top-$B$ resulting prefixes under $S(\cdot)$ are retained.
The next executed action is taken from the highest-scoring prefix.
We set $\mathcal{B} = 8$ by default.

Although beam search increases search width and can delay premature commitment, it still ranks trajectories solely using accumulated local scores and does not explicitly evaluate future outcomes.
As a result, it preserves the same step-wise scoring criterion as the greedy policy and remains vulnerable to the structural failure modes in Single-step greedy policy.

\noindent\textbf{Shallow Lookahead.}
The lookahead baseline augments step-wise decision making with limited forward simulation.
At each state $s_t$, for every candidate action $a \in \mathcal{A}(s_t)$, the agent performs rollouts of fixed depth $k$ ($k=2$ by default) using the true transition function to obtain a truncated trajectory
\(
(s_t, a, s_{t+1}, a_{t+1}, \dots, s_{t+k}).
\)
The cumulative return along this rollout is used as a proxy for the action's long-term utility,
\begin{equation}
    \hat{R}_k(s_t,a) = \sum_{i=0}^{k-1} \hat{r}(s_{t+i}, a_{t+i}, s_{t+i+1}),
\end{equation}
where subsequent actions are selected greedily according to the same step-wise score.

The selected action is then
\begin{equation}
    a_t^{\text{lookahead}} = \arg\max_{a \in \mathcal{A}(s_t)} \hat{R}_k(s_t,a).
\end{equation}

This baseline introduces explicit future evaluation but only over a limited horizon and without backward value propagation.
Consequently, early decisions remain dominated by truncated estimates and cannot be revised using information from deeper parts of the trajectory, distinguishing shallow lookahead from full planning mechanisms such as our proposed \method.

\subsection{Baselines}

\noindent\textbf{Inference-Only Methods.}
Standard Prompting (IO Prompt)~\cite{brown2020language} directly maps input prompts to outputs without exposing intermediate reasoning, which is proved to be superior in few-shot learning problems than traditional LMs.
Chain-of-Thought Prompting (CoT)~\cite{wei2022chain} generates a series of intermediate reasoning steps, enabling complex multi-step reasoning by decomposing complex problems into simpler sequential tasks.
Self-Consistency (SC)~\cite{wang2023selfconsistency} improves reasoning reliability by sampling multiple reasoning paths (e.g., CoT outputs) and selecting the most consistent answer.

\noindent\textbf{MCTS-based Methods.}
Tree-of-Thought (ToT)~\cite{yao2023tree} extends CoT by exploring multiple reasoning branches in a tree structure, allowing the model to evaluate, backtrack and select among diverse intermediate thoughts rather than following a linear reasoning path.
Reasoning-via-Planning (RAP)~\cite{hao2023reasoning} reformulates LLM reasoning as MCTS-based planning guided by an internal world model, where intermediate thought generation corresponds to planning over world states.
Language Agent Tree Search (LATS)~\cite{zhou2024language} proposes a unified framework where LLMs perform reasoning, acting and planning through a MCTS process, constructing the best action trajectories by integrating external feedback and self-reflection.

\noindent\textbf{KGQA Methods.}
KB-Binder~\cite{li2023few} is designed to address the heterogeneity of items from different knowledge bases, binding symbolic entities and relations into model's reasoning process to enable few-shot in-context learning over KGQA tasks.
Think-on-Graph (ToG)~\cite{sun2024thinkongraph} proposes a reasoning framework where LLMs perform multi-step reasoning by explicitly traversing a knowledge graph via beam search.
Think-on-Graph 2.0 (ToG 2.0)~\cite{ma2024think} extends the idea of ToG, improving graph exploration and pruning strategies to enhance scalability and consistency.
Reasoning-on-Graph (RoG)~\cite{luo2024reasoning} synergizes LLMs with knowledge graphs to leverage structural information within knowledge graphs and enable trustworthy reasoning.
Plan-on-Graph (PoG)~\cite{chen2024plan} is a self-correcting adaptive planning paradigm designed to achieve adaptive reasoning path exploration via question decomposition and sematic matching.
Graph-constrained Reasoning (GCR)~\cite{luo2025graphconstrained} bridges structured knowledge graphs with unstructured LLM reasoning via KG-Trie, constraining the decoding process and enabling scalable and faithful reasoning.
Deliberation over Priors (DP)~\cite{ma2025deliberation} adopts a progressive knowledge distillation algorithm and a reasoning-introspection strategy, integrating structural priors into LLMs and improving the faithfulness of reasoning path generation.
Reasoning-with-Trees (RwT)~\cite{shen2025reasoning} formulates KGQA as a discrete decision-making problem and adopts MCTS for iterative reasoning path refinement.
ProgRAG~\cite{park2025prograg} introduces a progressive retrieval and reasoning framework to mitigate hallucination. It progressively extends and refines partial reasoning paths by answering sub-questions decomposed from complex input questions.
Debate-on-Graph (DoG)~\cite{ma2025debate} leverages the interactive learning capabilities of LLMs and introduces a multi-role debate system to gradually simplify complex questions, reducing the influence of false-positive relations.
PathMind~\cite{liu2025pathmind} proposes a path prioritization mechanism and a "Retrieve-Prioritize-Reason" paradigm to selectively guide LLMs with important reasoning paths.

\subsection{Implementation Details}

\noindent\textbf{Knowledge Graph Question Answering.}
We evaluate all methods using four LLM backbones with varying capacities:
LLaMA-3.1-8B, LLaMA-3.1-70B, GPT-4o-mini, and GPT-4.
We report Hits@1 as the primary metric.
We use Think-on-Graph (ToG) \citep{sun2024thinkongraph} and Plan-on-Graph \citep{chen2024plan} frameworks as the agent backbone, which provide a standardized interface for graph traversal, action proposal, and LLM-based reasoning.
To isolate the effect of decision-making mechanisms, we keep all components fixed and vary only the planning logic used to select actions.
For \method, we fix the number of MCTS simulations per decision step and apply action-space pruning using a top-$k$ proposal mechanism ($k=8$ by default), and the similarity threshold for memory reuse as $\delta=0.9$, the memory size as $M=200$.
We set the rollout depth as $H=3$, the number of simulations as $S=16$, we call LLMs to evaluate rollouts rather than steps for efficiency,
the tree policy constant $c=1.4$.

\noindent\textbf{Tool Usage.}
We follow the standard ALFWorld in AgentGym setup \citep{xi2025agentgym} for tool-use tasks, using text-based observations and symbolic tool interfaces. Each task specifies a goal that can only be achieved by executing a correct sequence of tool calls. The environment is deterministic, and task success is evaluated by a predefined terminal condition. Importantly, certain locally plausible actions can lead to irreversible dead ends, after which no valid action sequence can satisfy the goal.
We instantiate two agent backbones, ReAct \citep{yao2022react} and Reflexion \citep{shinn2023reflexion}, which interleave reasoning and action in a step-wise manner. On top of each base agent, we implement our planning strategies.
All strategies share identical prompts, tool descriptions, and execution budgets, differing only in how future outcomes are evaluated and propagated during decision-making.
All experiments use GPT-4o-mini as the backbone language model. We evaluate each method on 200 tasks sampled from the environment. For each task, we record whether the agent successfully completes the goal and track the position of the first irreversible error, defined as the earliest step after which no valid continuation can lead to success. Reported metrics are averaged across tasks.
In addition to success rate, we report the first-error position to characterize how planning failures arise. Unlike success rate, which only reflects the final outcome, the first-error position provides a more direct measure of early decision quality in long-horizon tasks.

\subsection{Construction of Myopic Traps}

As shown in Section \ref{sec:diagnosis}, to systematically study early-stage failures induced by step-wise greedy decision making, we explicitly construct \emph{myopic traps} at the first decision step using only environment dynamics and oracle information, independent of any particular agent or planning algorithm.
For each task instance, we consider the initial state and enumerate all feasible successor actions available at the first step.
For each candidate action, we evaluate it from two complementary perspectives.
First, we compute its local step-wise score using the same evaluative signal that governs reasoning-based policies, reflecting how appealing the action appears under single-step decision making.
Second, we estimate its long-term utility by analyzing the resulting successor state in the oracle subgraph, including whether the ground-truth answer entity remains reachable and, when applicable, the length of the shortest valid path to any correct answer.

An action is labeled as a myopic trap if it satisfies both of the following conditions:
(1) it ranks among the top actions according to the local step-wise score, and
(2) it leads to a region of the state space with substantially worse long-term prospects than at least one alternative action, such as losing reachability to the answer or requiring significantly longer solution paths.
To ensure a meaningful causal comparison, we additionally require that for each instance there exists at least one non-trap action whose successor state admits a valid solution path to the correct answer.
This guarantees that failure under trap selection is not caused by task infeasibility, but by the discrepancy between local and long-term evaluation.
Importantly, trap construction is performed offline using the explicit state space, transition structure, and oracle subgraphs, and does not depend on the behavior of any LLM, planning strategy, or learned policy.
This procedure isolates myopic traps as an intrinsic property of the environment and evaluation signal, enabling controlled analysis of how early locally optimal decisions are amplified into long-horizon planning failures.

%% file: appendix/experimental_results.tex
\section{Supplementary Experiments}

\subsection{Comprehensive Planning Mechanism Analysis}
\label{sec:mechanism}

\input{table/diagnosis-full}

This section provides a detailed mechanism-level analysis that complements the aggregated results reported in Table~\ref{tab:diagnosis-summary}.
All experiments are conducted under the same oracle-structure setting described in Section~5, using identical state representations, transition dynamics, and evaluative signals.
Results are averaged over multiple random seeds.

\noindent\textbf{Metrics.}
We report three metrics designed to characterize long-horizon decision dynamics:

\begin{itemize}
    \item \textbf{(Constructed) Trap@1}: the probability that the agent selects a constructed myopic trap at the first decision step;
    \item \textbf{First-Error Step}: the average decision step at which the trajectory first deviates from any optimal path;
    \item \textbf{Recovery@First-Error}: the probability of reaching the correct answer after the first deviation occurs.
\end{itemize}

Together, these metrics quantify early-stage decision bias, error accumulation, and post-error recoverability.

\noindent\textbf{Aggregated comparison.}
As shown in Table~\ref{tab:diagnosis-summary}, single-step greedy decision making exhibits a high rate of early trap selection (55.6\%) and deviates from optimal trajectories within the first two steps on average (1.6).
Once an error occurs, recovery is rare (5.4\%), indicating that step-wise greedy policies are prone to early irreversible commitment.
Beam search delays the first deviation slightly (average step 2.0) but substantially increases the probability of selecting a myopic trap at the first step (71.9\%).
This behavior is consistent with its reliance on accumulated step-level surrogate scores, which amplify local relevance signals without encoding long-term feasibility.
Consequently, recovery remains limited (11.4\%), and long-horizon failures persist.
Shallow lookahead significantly reduces early trap selection (23.6\%) and improves both the position of the first error and post-error recovery (22.4\%).
These results indicate that limited forward evaluation partially mitigates local decision bias.
However, recovery remains constrained, suggesting that truncated rollouts without backward value propagation are insufficient to fully stabilize long-horizon behavior.
\method achieves the strongest performance across all three metrics.
It reduces Trap@1 to 17.8\%, postpones the first deviation to an average of 3.2 steps, and increases recovery probability to 29.7\%.
This improvement is consistent with the use of trajectory-level value propagation, which allows downstream outcomes to revise earlier action preferences and mitigates irreversible early commitment.

\noindent\textbf{Per-dataset analysis.}
Table~\ref{tab:diagnosis-full} reports the same metrics separately for CWQ, WebQSP, and GrailQA.
Across all datasets, \method consistently attains the lowest trap selection rate, the latest first-error position, and the highest recovery probability.
Absolute values vary across benchmarks, reflecting differences in planning horizon length and structural complexity.
In particular, CWQ and GrailQA exhibit higher Trap@1 rates and lower recovery probabilities for all methods, indicating more severe long-horizon challenges.
Nevertheless, the relative advantages of \method remain stable across datasets, confirming that the aggregated trends in Table~\ref{tab:diagnosis-summary} do not arise from any single benchmark but reflect a consistent improvement in decision dynamics across diverse KGQA settings.

\subsection{Comprehensive Ablation Study}

Table~\ref{tab:ablation-full} and Figure \ref{fig:ablation-tradeoff} reports the ablation results of \method on CWQ, WebQSP, and GrailQA, including both task performance (Hits@1) and computational cost measured by the number of tokens consumed during planning.
We study two design components that are not part of the planning objective itself but affect how efficiently planning is carried out: Action Pruning and Trajectory Memory.

\input{table/ablation-full.tex}

\begin{figure}[!t]
    \centering
    \includegraphics[width=\linewidth]{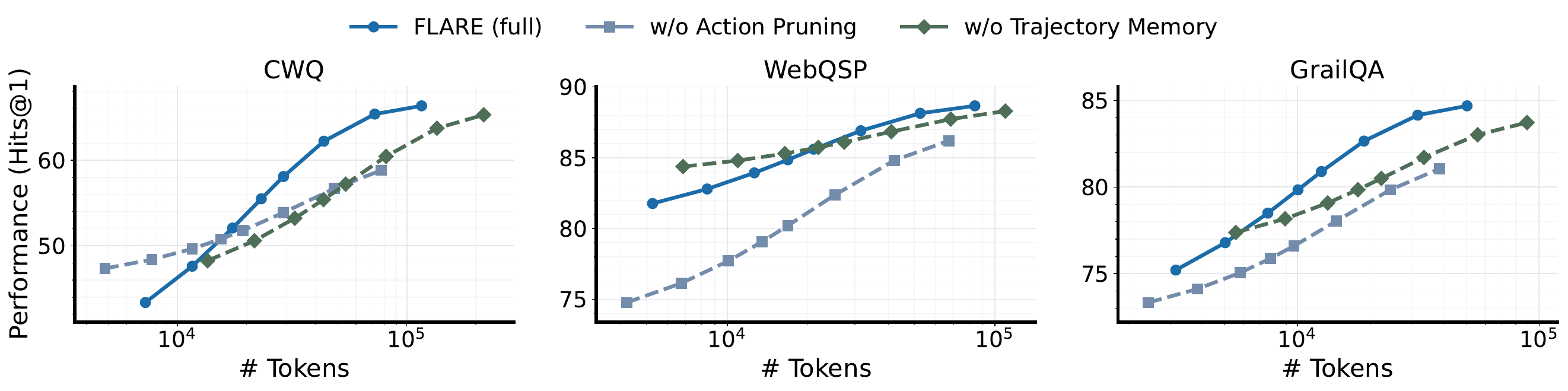}
    \caption{
        \textbf{Performance versus token budget for FLARE and its ablations.}
        FLARE achieves consistently higher accuracy across budgets, while removing action pruning or trajectory memory degrades performance and sample efficiency, demonstrating their complementary roles in efficient long-horizon planning.
    }
    \label{fig:ablation-tradeoff}
\end{figure}

\noindent\textbf{Action Pruning.}
Action pruning limits the branching factor during tree expansion by restricting the set of candidate actions evaluated at each state.
From Table~\ref{tab:ablation-full}, removing action pruning under the default configuration leads to a noticeable drop in Hits@1 across all datasets, indicating that unconstrained expansion degrades effective exploration under limited budgets.
When additional computation is provided to recover the same performance level as the full model (the Hits@1-aligned setting), token usage increases substantially (e.g., more than doubling on CWQ and tripling on WebQSP), demonstrating that similar accuracy can only be achieved at significantly higher cost without pruning.
Figure~\ref{fig:ablation-tradeoff} further reveals this effect in terms of performance--budget dynamics.
Across all three benchmarks, the curve without action pruning improves more slowly with increasing token budget and remains consistently below the full model at comparable cost.
This indicates that action pruning primarily shifts the efficiency frontier, enabling the planner to reach higher accuracy with fewer evaluations by preventing budget from being spent on low-quality branches.
Importantly, pruning does not alter the planning objective or evaluation signal, but improves how computational resources are allocated during lookahead.



\noindent\textbf{Trajectory Memory.}
Trajectory memory caches and reuses previously evaluated trajectories during rollout and value propagation, reducing redundant evaluation.
As shown in Table~\ref{tab:ablation-full}, disabling trajectory memory increases token consumption even when performance is similar, reflecting repeated evaluation of similar futures.
Under the token-aligned setting, where the memory-free variant is constrained to use the same computational budget as the full model, its performance drops consistently across all datasets, confirming that the performance gap cannot be explained solely by reduced computation.
The dynamic curves in Figure~\ref{fig:ablation-tradeoff} further illustrate this effect.
Removing trajectory memory shifts the entire performance--budget curve downward: at any fixed token budget, the memory-free variant attains lower Hits@1, and additional computation is required to approach the accuracy of the full model.
This behavior indicates that trajectory memory improves not only efficiency, but also the effectiveness of value accumulation under limited budgets, by stabilizing trajectory-level estimates and reducing noise in value propagation.

\noindent\textbf{Summary.}
Together, Table~\ref{tab:ablation-full} and Figure~\ref{fig:ablation-tradeoff} show that the two components affect the performance--cost trade-off in complementary ways.
Action pruning primarily improves exploration efficiency, shifting the Pareto frontier by enabling higher performance at lower computational cost.
Trajectory memory improves evaluation efficiency, allowing the planner to convert a given computational budget into more reliable value estimates and higher final accuracy.
Neither component changes the planning objective itself; instead, both modify the dynamics by which performance accumulates as computation increases, resulting in a consistently superior efficiency--performance frontier for the full \method model across all datasets.

\subsection{Failure Pattern Analysis}
\label{sec:failure-patterns}

\begin{figure*}[!t]
    \centering
    \includegraphics[width=\linewidth]{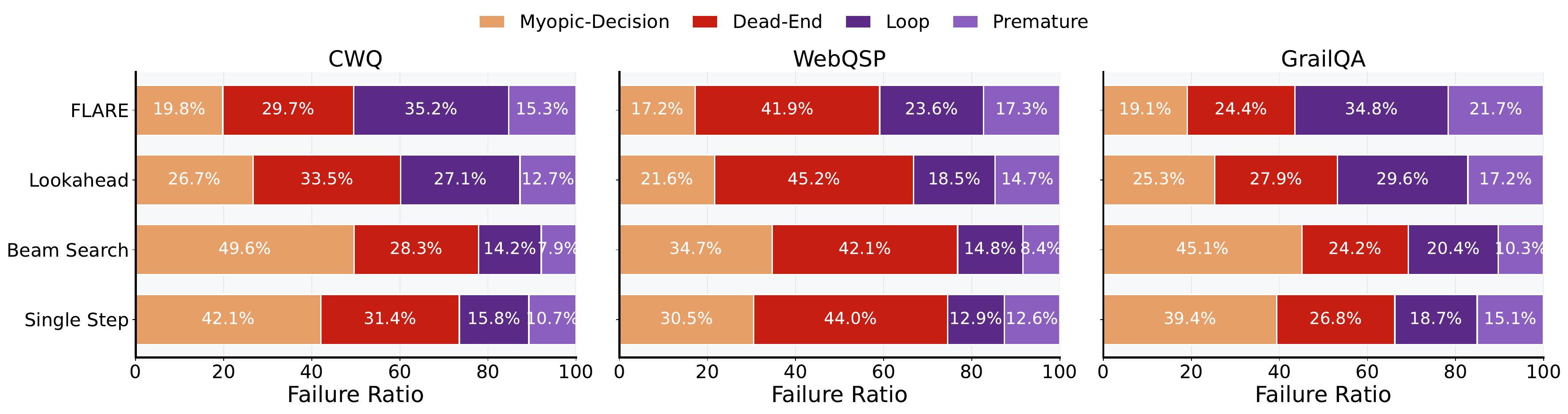}
    \caption{
        \textbf{Distribution of failure types under different planning strategies.}
        Unsuccessful trajectories are categorized into four planning-related failure modes: myopic deviation, dead-end, loop, and premature termination.
        Results are computed from 50 randomly sampled failed trajectories per method and dataset.
    }
    \label{fig:failure-patterns}
\end{figure*}

This section provides a diagnostic analysis of failure patterns exhibited by different planning strategies.
The goal is to characterize the dominant mechanisms underlying unsuccessful trajectories and to identify which failure modes are mitigated by explicit planning and which remain.

\noindent\textbf{Failure categorization.}
For each method and dataset, we randomly sample and review 50 failed trajectories.
Each trajectory is manually assigned to one of the following categories based on its dominant failure mechanism:

\begin{itemize}
    \item \textbf{Myopic deviation}: the agent selects an early action that is locally appealing but eliminates or severely degrades long-term solvability;
    \item \textbf{Dead-end}: the agent reaches a state from which no valid path to the ground-truth answer exists in the oracle subgraph;
    \item \textbf{Loop}: the agent repeatedly revisits previously explored states without making progress toward the answer;
    \item \textbf{Premature termination (Premature)}: the agent halts execution before completing a valid reasoning chain.
\end{itemize}

Figure~\ref{fig:failure-patterns} reports the empirical distribution of these categories on CWQ, WebQSP, and GrailQA.

\noindent\textbf{Observed trends.}
Across all three datasets, step-wise greedy and beam-based strategies are dominated by myopic deviations, indicating that early locally optimal but globally suboptimal actions constitute the primary source of failure under step-wise or width-based decision making.
Shallow lookahead substantially reduces the frequency of myopic deviations, consistent with its use of limited future evaluation.
\method further reduces this category to the lowest level among all methods, reflecting improved early-stage decision quality.
As the prevalence of myopic deviations decreases, the relative frequency of loop and premature termination failures increases.
These failure modes are associated with limitations in exploration control and termination criteria rather than with incorrect early action selection.
Dead-end failures remain present but constitute a smaller fraction across all planning-based methods.

\noindent\textbf{Discussion.}
This analysis suggests that explicit lookahead with value propagation primarily addresses failures caused by early decision bias.
The remaining dominant failure modes under \method arise from challenges in exploration efficiency and termination control, which are not directly targeted by the current planning mechanism.
Overall, the results indicate that \method alters the qualitative structure of planning failures, shifting them from early irreversible deviations toward later-stage procedural issues such as looping and premature termination.

%% file: table/diagnosis-full.tex
\begin{table}[!t]
    \centering
    \caption{
        \textbf{Mechanism-level analysis of long-horizon decision dynamics on CWQ, WebQSP, and GrailQA.}
        We report the probability of selecting a constructed myopic trap at the first step (Trap@1; lower is better),
        the average step at which the trajectory first deviates from any optimal path (First-Error Step; higher is better),
        and the probability of reaching the correct answer after the first error occurs (Recovery@First-Error; higher is better).
        Results are computed under the oracle-structure setting and averaged over multiple runs.
    }
    \label{tab:diagnosis-full}
    \begin{tabular}{clcccc}
        \toprule
        \textbf{Dataset}          & \textbf{Method} & \textbf{Trap@1 $\downarrow$} & \textbf{First-Error Step $\uparrow$} & \textbf{Recovery@First-Error $\uparrow$} \\ \midrule
                                  & Single Step     & 56.3                         & 1.6                                  & 5.2                                      \\
                                  & Beam Search     & 72.4                         & 2.0                                  & 11.1                                     \\
                                  & Lookahead       & 24.7                         & 2.8                                  & 21.4                                     \\
        \multirow{-4}{*}{CWQ}     & \method         & \textbf{17.8}                & \textbf{3.4}                         & \textbf{29.6}                            \\ \midrule
                                  & Single Step     & 51.8                         & 1.8                                  & 6.1                                      \\
                                  & Beam Search     & 69.9                         & 2.2                                  & 12.9                                     \\
                                  & Lookahead       & 18.9                         & 2.9                                  & 26.4                                     \\
        \multirow{-4}{*}{WebQSP}  & \method         & \textbf{14.6}                & \textbf{3.2}                         & \textbf{31.7}                            \\ \midrule
                                  & Single Step     & 58.7                         & 1.5                                  & 4.8                                      \\
                                  & Beam Search     & 73.5                         & 1.9                                  & 10.2                                     \\
                                  & Lookahead       & 27.1                         & 2.6                                  & 19.5                                     \\
        \multirow{-4}{*}{GrailQA} & \method         & \textbf{20.9}                & \textbf{3.1}                         & \textbf{27.8}                            \\
        \bottomrule
    \end{tabular}
\end{table}

%% file: table/ablation-full.tex
\begin{table}[!t]
\centering
\caption{
\textbf{Ablation of action pruning and trajectory memory in FLARE.}
Results on CWQ, WebQSP, and GrailQA show that removing action pruning significantly reduces performance under similar budgets and requires substantially more tokens to recover accuracy, while removing trajectory memory mainly increases token usage and lowers efficiency under matched budgets.
}
\label{tab:ablation-full}
\resizebox{\linewidth}{!}{
\begin{tabular}{l | cccccccc}
\toprule
& \multicolumn{2}{c}{\textbf{CWQ}} & \multicolumn{2}{c}{\textbf{WebQSP}} & \multicolumn{2}{c}{\textbf{GrailQA}} & \multicolumn{2}{c}{\textbf{Overall}} \\
\cmidrule(lr){2-3}\cmidrule(lr){4-5}\cmidrule(lr){6-7}\cmidrule(lr){8-9}
\textbf{Methods}
& \textbf{Hits@1} & \textbf{\# Tokens}
& \textbf{Hits@1} & \textbf{\# Tokens}
& \textbf{Hits@1} & \textbf{\# Tokens}
& \textbf{Hits@1} & \textbf{\# Tokens} \\
\midrule
\method (Full)
& 58.1 & 29
& 85.6 & 21k
& 80.9 & 13k
& 74.9 & 21k \\
\midrule
\multicolumn{9}{l}{\textbf{\textit{(a) Action Pruning (performance-aligned)}}} \\ \midrule
w/o Pruning
& 51.8 & 19
& 80.2 & 17k
& 76.6 & 10k
& 69.5 & 15k \\
w/o Pruning (Hits@1-Aligned)
& 58.8 & 77
& 86.2 & 67k
& 81.1 & 39k
& 75.4 & 61k \\
\midrule
\multicolumn{9}{l}{\textbf{\textit{(b) Trajectory Memory (token-aligned)}}} \\
\midrule
w/o Memory
& 57.2 & 54
& 86.1 & 27k
& 80.5 & 22k
& 74.6 & 34k \\
w/o Memory (Token-Aligned)
& 53.2 & 32
& 85.7 & 22k
& 79.1 & 13k
& 72.7 & 22k \\
\bottomrule
\end{tabular}
}
\end{table}